\documentclass{article}

\usepackage[preprint]{neurips_2026}

\usepackage[utf8]{inputenc} 
\usepackage[T1]{fontenc}    
\usepackage{hyperref}       
\usepackage{url}            
\usepackage{booktabs}       
\usepackage{amsfonts}       
\usepackage{nicefrac}       
\usepackage{microtype}      
\usepackage{xcolor}         

\usepackage{url}            
\usepackage{booktabs}       
\usepackage{amsfonts}       
\usepackage{nicefrac}       
\usepackage{microtype}      

\usepackage{algorithm}

\usepackage{amssymb}		
\usepackage{graphicx}		
\usepackage{amsmath}
\usepackage{amsthm}
\usepackage{algorithmic}
\usepackage{comment}
\usepackage{multirow}
\usepackage{enumitem}
\usepackage{subcaption}
\usepackage{tikz}
\usepackage[dvipsnames,table,xcdraw]{xcolor}

\usepackage{mathrsfs}

\usepackage{amsmath}  
\usepackage{xspace}   

\usepackage{pifont}
\usepackage{float}

\usepackage{amsmath,amssymb,mathtools}

\newcommand{\KL}{\mathrm{KL}}         

\usepackage[nameinlink]{cleveref}  

\newtheorem{theorem}{Theorem}

\newtheorem{lemma}{Lemma}

\newtheorem{definition}{Definition}

\newtheorem{corollary}{Corollary}

\crefname{section}{Section}{Sections}
\Crefname{section}{Section}{Sections}

\crefname{subsection}{Section}{Sections}
\Crefname{subsection}{Section}{Sections}

\newcommand{\TV}{\mathrm{TV}}

\newcommand{\Reach}{\operatorname{Reach}}

\newcommand{\one}{\mathbb{1}}

\title{The Role of Generator Access in Autoregressive Post-Training}

\author{%
  Amit Kiran Rege \\
  Department of Computer Science\\
  University of Colorado Boulder\\
  Boulder, Colorado 80309 \\
  \texttt{amit.rege@colorado.edu} \\
}

\begin{document}

\maketitle

\begin{abstract}
We study how generator access constrains autoregressive post-training. The central question is whether the learner is confined to fresh root-start rollouts or can return to previously built prefixes and query the next-token rule there. In the root-start regime, output sampling, generated-token log probabilities, top-$k$ reports, and full next-token distributions along sampled trajectories all reduce to one canonical experiment, limited by the on-policy probability of reaching informative prefixes. Weak prefix control breaks this barrier, and once control is available, richer observations such as conditional sampling or logits can outperform top-$1$ access. Changing only the generator interface creates an exponential gap for KL-regularized outcome-reward post-training.
\end{abstract}

\section{Introduction}
\label{sec:intro}

Large language models have transformed modern AI, and recent reasoning-oriented systems push that progress further through post-training, reinforcement learning, and inference-time search \cite{OpenAIo1SystemCard2024,DeepSeekR1Guo2025,HuangEtAl2025BestOfN}. At the same time, a recurring empirical theme is that standard post-training often sharpens trajectories that are already present in the base model, rather than uncovering genuinely new ones \cite{YueEtAl2025RLVR,KaranDu2025Sampling,TanEtAl2026LED}. That raises a basic question which is easy to state but has not been cleanly separated in existing theory. When post-training runs into a barrier, how much of that barrier comes from the base distribution itself, and how much comes from the way the learner is allowed to interrogate the generator?

A simple thought experiment already shows why this matters. Imagine a reasoning task whose successful solution requires a long chain of correct intermediate decisions. Along the one correct chain, the base model gives the right next token only a mild advantage over its competitors. After a single wrong turn, however, the distribution becomes uninformative. In that situation, ordinary sampling has to survive the entire chain on policy. Generated-token log probabilities, top-$k$ reports, or even full next-token distributions on visited prefixes do not change that basic geometry: they still reveal information only along whatever trajectory the model happened to generate. By contrast, if the learner can return to a partially constructed prefix and inspect the next-token rule there, then it can test the $t$th step directly instead of re-earning the first $t-1$ steps every time. This gap is about the interface to the generator and not about a more sophisticated optimizer.

This paper studies that interface through two notions. Prefix control asks which prefixes the learner may place the generator at before making a query. Prefix observation asks what the learner sees once that prefix has been fixed. The first distinction is about where the learner may look; the second is about what the learner sees there. Our main claim is that, for autoregressive post-training, these two distinctions do not sit on equal footing. Prefix control is the first boundary. Observation richness becomes the second boundary only after control has been granted.

\begin{table}[t]
\centering
\begin{tabular}{p{0.24\linewidth}p{0.31\linewidth}p{0.31\linewidth}}
\toprule
 & Limited observation & Rich observation \\
\midrule
No prefix control &
Output-only sampling &
Generated-token log probabilities, top-log-probability lists, or full next-token distributions on visited prefixes \\
Prefix control &
Top token at a chosen prefix &
Chosen-prefix sampling, chosen-prefix logits, or teacher-forced sequence scores \\
\bottomrule
\end{tabular}
\caption{The taxonomy studied in this paper. The main point is that the horizontal distinction comes first. Once the learner is confined to root-start rollouts, richer trajectory-local observations do not overcome the need to reach informative prefixes on policy.}
\label{tab:taxonomy}
\end{table}

The paper develops this point in four steps. First, we characterize the no-reset regime, where each generator query begins with a fresh rollout from the root and may depend only on the sampled trajectory and the next-token distributions encountered along that trajectory. In this regime there is a single canonical experiment. Output-only sampling, generated-token log probabilities, top-$k$ reports on sampled tokens, and even full next-token distributions on sampled prefixes are all just randomized post-processings of the same object. Second, once the learner is allowed even weak local reset, meaning that it may revisit previously queried prefixes and extend them one token at a time, the picture changes sharply: the universal no-reset barrier disappears on a hidden-path family. Third, after control has been supplied, observation richness becomes meaningful again. On a branching family, top-$1$ access can be much weaker than chosen-prefix sampling or logits. Fourth, this generator-side distinction can be lifted into the standard post-training objective. Holding the reward model, the KL regularizer, and the policy class fixed, changing only generator access can create an exponential gap in the number of generator queries required for outcome-reward post-training.

This perspective complements two lines of recent work. One line fixes the access model and shows that coverage of good responses under the base model governs the difficulty of sampling-style alignment and inference-time selection \cite{FosterMhammediRohatgi2025,HuangEtAl2025BestOfN,ChenEtAl2025Coverage,HuangEtAl2024Sharpening}. Another line fixes generator access and studies the feedback channel, showing for example that process feedback can avoid outcome-reward barriers \cite{MousaviHosseiniErdogdu2026}. Our contribution is to isolate the remaining axis. We ask what changes when the reward and update class are held fixed and only the generator-side interface is allowed to vary.

The rest of the paper is organized as follows. Section~\ref{sec:setup} formalizes prompt-conditioned autoregressive generators, the prefix-tree view, and the access models studied here. Section~\ref{sec:no-reset} proves the collapse of the no-reset regime and the resulting reachability barrier. Later sections show how weak local reset breaks that barrier, why branching makes observation matter again once control is available, and how the same boundary reappears in KL-regularized post-training.

\section{Setup}
\label{sec:setup}

\subsection{Prompt-conditioned autoregressive generators}

Let $\mathcal X$ be a prompt space with prompt distribution $\rho$. Fix a finite vocabulary $\Sigma=\{1,2,\dots,K\}$ with $K\ge 2$ and a horizon $H\ge 1$. For $t\ge 0$, let $\Sigma^t$ denote the set of strings of length $t$, write $\Sigma^{<H}:=\bigcup_{t=0}^{H-1}\Sigma^t$ and $\Sigma^{\le H}:=\bigcup_{t=0}^{H}\Sigma^t$, and let $\varnothing$ denote the unique element of $\Sigma^0$.

A prompt-conditioned autoregressive generator is a family of next-token distributions
\[
M(\cdot\mid x,p)\in \Delta(\Sigma), \qquad x\in\mathcal X,\; p\in\Sigma^{<H}.
\]
Given a prompt $x$, a response $Y=(Y_1,\dots,Y_H)\in\Sigma^H$ is generated sequentially by drawing $Y_t\sim M(\cdot\mid x,Y_{<t})$ for $t=1,\dots,H$. Equivalently,
\[
\Pr_M(Y=y\mid x)=\prod_{t=1}^{H} M(y_t\mid x,y_{<t}) \qquad \text{for every } y\in\Sigma^H.
\]

The structural results hold at a fixed prompt; the prompt distribution enters only in the last section, where we study a prompt-averaged post-training objective. The single-hard-prompt case corresponds to $\rho$ being a point mass.

\subsection{Prefix trees and generator access}

For a fixed prompt $x$, autoregressive generation can be viewed as a depth-$H$ tree whose internal nodes are prefixes in $\Sigma^{<H}$. From a prefix $p$, each token $a\in\Sigma$ leads to the child $pa$, and the generator attaches to $p$ the next-token distribution $M(\cdot\mid x,p)$. In this picture, a complete response is a root-to-leaf path.

This viewpoint makes the access distinctions concrete. A learner might only be allowed to start at the root, sample a rollout, and inspect whatever prefixes that rollout happened to visit. Or it might be allowed to return to a previously constructed prefix and ask for the next-token rule there. The first kind of interaction is the analogue of ordinary episodic access in reinforcement learning; the second is the analogue of reset-like access \cite{MhammediFosterRakhlin2024,RohatgiFoster2025}. For autoregressive models, the natural state space is the prefix tree itself.

We separate generator access into two parts. Prefix control asks which prefixes the learner may interrogate. Prefix observation asks what the learner receives once the relevant prefix has been fixed. A standard API that returns output tokens together with generated-token log probabilities has richer observation than plain sampling, but no extra control: both are tied to a single root-start rollout. A chosen-prefix sampling oracle and a chosen-prefix logit oracle have the same control but different observations. Teacher-forced sequence scoring is score-like, but it lies on the control-rich side because the learner is allowed to choose the entire completion whose probability is evaluated.

The weakest regime studied in the paper is no-reset access. Informally, one query chooses a parameter $q$, draws a fresh root-start rollout $Y\sim M(\cdot\mid x)$, forms the visited-prefix distributions $\mu_t:=M(\cdot\mid x,Y_{<t})$, and returns a reply that may depend on $q$, on the sampled trajectory $Y$, and on the visited-prefix distributions $(\mu_1,\dots,\mu_H)$, but not on any unvisited prefix. This covers output-only sampling, generated-token log probabilities, sampled-token top-$k$ reports, and even full next-token distributions on sampled prefixes.

The stronger regime used later is weak local reset. A sequence of queried prefixes $p_1,\dots,p_q\in\Sigma^{<H}$ obeys the local-reset discipline if $p_1=\varnothing$ and, for every round $r\ge 2$, the new query $p_r$ is either a previously queried prefix or a one-token extension of a previously queried prefix. The learner is not allowed to teleport to an unrelated deep node, but it is allowed to build a prefix one token at a time and to revisit prefixes it has already constructed.

Once local reset is available, we consider three concrete chosen-prefix interfaces. Under $\mathsf{PrefixSample}$, a query prefix $p$ returns an independent draw $A\sim M(\cdot\mid x,p)$. Under $\mathsf{PrefixTop}$, a query prefix $p$ returns the unique most likely next token under $M(\cdot\mid x,p)$ when that maximizer is unique, and returns a distinguished symbol $\bot$ otherwise. Under $\mathsf{PrefixLogit}_{\xi}$, a query prefix $p$ returns a vector $\widetilde \ell_p\in\mathbb R^\Sigma$ satisfying $\|\widetilde \ell_p-\log M(\cdot\mid x,p)\|_\infty\le \xi$. We write $\mathsf{PrefixLogit}$ for the exact case $\xi=0$.

We also include teacher-forced sequence scoring. Under $\mathsf{SeqScore}_\xi$, a query completion $y\in\Sigma^H$ returns a number $\widetilde s(y)$ satisfying
\[
\bigl|\widetilde s(y)-\log \Pr_M(Y=y\mid x)\bigr|\le \xi.
\]
Operationally, this is just teacher forcing followed by reading off the sequence log-likelihood. It is not a literal chosen-prefix oracle, but it still lies on the control-rich side of the taxonomy because the learner is allowed to evaluate a counterfactual completion of its choice rather than waiting for that completion to arise on policy.

These interfaces are close to what current systems expose. Generated-token log probabilities and top-log-probability lists appear in widely used APIs, while open-weight causal language models expose next-token logits and sequence scores through ordinary forward passes \cite{OpenAIResponsesAPI2026,OpenAIChatCompletionsAPI2026,HuggingFaceGenerationDocs2026,HuggingFaceGPT2Docs2026,HuggingFacePerplexityDocs2026}. The point of the taxonomy is not to invent exotic oracles. It is to organize access patterns that already appear in practice.

In practical terms, weak local reset corresponds to a very modest form of counterfactual control. The learner does not get to jump to an arbitrary deep prefix. It can only keep a partial continuation, revisit it, and extend it one token at a time. This is the prefix-tree analogue of stepping back to a previously constructed partial solution and resuming from there, rather than resampling the entire trajectory from the beginning.

\subsection{KL-regularized outcome-reward post-training}

The last part of the paper returns from access models to a standard post-training objective. Let $r:\mathcal X\times\Sigma^H\to\mathbb R$ be an outcome reward, let $M(\cdot\mid x)$ be the base generator, and let $\pi(\cdot\mid x)$ be a prompt-conditioned policy. For $\beta>0$, define
\[
J_\beta(\pi):=
\mathbb E_{x\sim \rho}\left[\mathbb E_{y\sim \pi(\cdot\mid x)} r(x,y)-\beta\,\KL\!\left(\pi(\cdot\mid x)\,\|\,M(\cdot\mid x)\right)\right].
\]
The final theorem keeps this reward model, this KL-regularized objective, and the policy class fixed, and varies only the generator-side interface. The resulting separation shows that the access taxonomy is not merely a statement about synthetic search problems. It already changes the complexity of a standard post-training objective.

\section{The no-reset regime}
\label{sec:no-reset}

The no-reset regime is the root-start side of the taxonomy. Each query begins with a fresh rollout from the root, and whatever the learner sees is constrained to the prefixes visited by that rollout. At first sight, this leaves many different interfaces: the sampled completion, generated-token log probabilities, a top-$k$ list, or even the full next-token distributions on sampled prefixes. Statistically, however, these are all equivalent.

Throughout this section we fix a prompt $x\in\mathcal X$. Prompt-distributed versions follow by averaging over $x\sim\rho$.

\begin{definition}
\label{def:no-reset}
A generator experiment at prompt $x$ is called no-reset if one query is answered as follows. The learner chooses a query parameter $q$. The oracle then draws a fresh rollout $Y=(Y_1,\dots,Y_H)\sim M(\cdot\mid x)$, forms the visited-prefix distributions $\mu_t:=M(\cdot\mid x,Y_{<t})$ for $t=1,\dots,H$, and returns a reply drawn from a kernel that may depend on $q$, on $Y$, and on $(\mu_1,\dots,\mu_H)$, but not on any unvisited prefix.
\end{definition}

Definition~\ref{def:no-reset} is deliberately broad. It allows arbitrary measurable summaries of the sampled trajectory and of the next-token distributions encountered along that trajectory, including arbitrary internal randomization. What it forbids is counterfactual inspection of prefixes that the rollout did not visit.

The canonical experiment in this regime simply returns everything that a no-reset interface is ever allowed to depend on.

\begin{definition}
\label{def:pathfull}
For a fixed prompt $x$, let $\mathsf{PathFull}_x$ denote the experiment that returns
\[
W_x:=(Y,\mu_1,\dots,\mu_H),
\]
where $Y\sim M(\cdot\mid x)$ and $\mu_t=M(\cdot\mid x,Y_{<t})$ for $t=1,\dots,H$.
\end{definition}

\begin{theorem}
\label{thm:no-reset-collapse}
Fix a prompt $x$. An experiment is no-reset if and only if it is a randomized post-processing of $\mathsf{PathFull}_x$. In the usual comparison-of-experiments sense, $\mathsf{PathFull}_x$ is therefore the maximal no-reset experiment.
\end{theorem}

Once the learner is confined to one root-start rollout per query, there is no deeper hierarchy among trajectory-local interfaces. Output-only sampling, generated-token log probabilities, sampled-token top-$k$ reports, and full next-token distributions on sampled prefixes all sit inside the same experiment class; richer replies simply retain more information from the same trajectory.

This means that a lower bound on $\mathsf{PathFull}_x$ applies simultaneously to every trajectory-local interface. The key quantity governing that bound is on-policy reachability.

\begin{definition}
\label{def:reachability}
For a prefix set $U\subseteq\Sigma^{<H}$, define the reachability event
\[
E_U(x):=\bigl\{\exists\, t\in\{1,\dots,H\}\;:\; Y_{<t}\in U\bigr\}, \qquad Y\sim M(\cdot\mid x),
\]
and the corresponding prefix reachability
\[
\Reach_M(x,U):=\Pr\bigl(E_U(x)\bigr).
\]
\end{definition}

The interpretation is immediate: $\Reach_M(x,U)$ is the probability that an on-policy rollout ever visits the informative region $U$.

Now suppose two generators $M(\cdot\mid x)$ and $M'(\cdot\mid x)$ agree outside a prefix set $U$, meaning that $M(\cdot\mid x,p)=M'(\cdot\mid x,p)$ for every $p\notin U$. If a rollout never enters $U$, then the two models generate exactly the same transitions along that rollout, and the canonical no-reset reply is identical in the two worlds. The only way to tell the models apart is to reach $U$ on policy.

\begin{theorem}
\label{thm:reachability-barrier}
Fix a prompt $x$ and a prefix set $U\subseteq\Sigma^{<H}$. Let $M(\cdot\mid x)$ and $M'(\cdot\mid x)$ agree outside $U$. Let $A$ be any randomized adaptive algorithm that makes at most $q$ queries to any no-reset experiment at prompt $x$, and let $T_q^A(M,x)$ denote its full transcript under $M$. Then
\[
\TV\!\Bigl(\mathcal L\bigl(T_q^A(M,x)\bigr),\mathcal L\bigl(T_q^A(M',x)\bigr)\Bigr)\le q\,\Reach_M(x,U).
\]
The same reachability value appears for both models because agreement outside $U$ implies $\Reach_M(x,U)=\Reach_{M'}(x,U)$.
\end{theorem}

Theorem~\ref{thm:reachability-barrier} is the core lower bound on the no-reset side. It says that the real bottleneck is not whether the API reveals text, scores, or full distributions along a sampled path. The bottleneck is that the learner cannot inspect prefixes that the base model fails to reach on its own.

A direct consequence is that the standard trajectory-local interfaces used in practice all obey the same barrier.

\begin{corollary}
\label{cor:standard-no-reset}
Under the assumptions of Theorem~\ref{thm:reachability-barrier}, the same bound holds for output-only sampling, for output together with generated-token log probabilities, for output together with sampled-token top-log-probability lists or top-$k$ reports, and for replies that include the full next-token distributions on sampled prefixes.
\end{corollary}

This corollary is worth pausing over. Generated-token log probabilities are often treated as substantially stronger than sampling, and in a narrow sense they are. They reveal more about the prefixes that were actually visited. The point of Theorem~\ref{thm:reachability-barrier} is that this extra information does not change the need to reach the informative part of the prefix tree in the first place.

The next sections confirm that this is indeed a control barrier. On the hidden-path family, every no-reset interface faces the same exponential bottleneck, while weak local reset breaks it. Once that boundary is crossed, a second distinction emerges: on a branching family, richer observations such as conditional sampling or logits outperform top-$1$ access even at chosen prefixes.

\section{Escaping the barrier with weak local reset}
\label{sec:hidden-path}

The previous section shows that every no-reset interface is governed by the same on-policy reachability quantity. The question is whether this bottleneck comes from the absence of prefix control specifically, or reflects a more general hardness. The hidden-path family answers this by isolating the long-horizon cost as directly as possible.

We use a hidden path. This family is meant to model a fragile chain of reasoning: there is one long sequence of mildly favored correct steps, and once the rollout leaves that sequence the model stops providing useful guidance. Ordinary sampling then has to survive the whole chain on policy, while local reset lets the learner inspect one step at a time. In this section and the next, all witness families are conditioned on a single fixed prompt, which we suppress from the notation.

Fix a signal parameter $\lambda>0$ and define
\[
p_+ := \frac{e^\lambda}{e^\lambda+K-1}, \qquad
p_- := \frac{1}{e^\lambda+K-1}, \qquad
\Delta := p_+ - p_-.
\]
Thus $p_+>1/K>p_-$ and $\Delta>0$.

\begin{definition}
\label{def:hidden-path}
For each hidden string $z=(z_1,\dots,z_H)\in\Sigma^H$, the hidden-path generator $M_z$ is defined by
\[
M_z(a\mid p)=
\begin{cases}
p_+, & \text{if } p=z_{<t} \text{ for some } t\in\{1,\dots,H\} \text{ and } a=z_t,\\
p_-, & \text{if } p=z_{<t} \text{ for some } t\in\{1,\dots,H\} \text{ and } a\neq z_t,\\
1/K, & \text{if } p \text{ is not a prefix of } z.
\end{cases}
\]
\end{definition}

Along the hidden path, the correct next token is only mildly preferred. Away from the hidden path, the model is uniform and carries no information about $z$. That is exactly the feature we want: the family isolates the cost of repeatedly re-reaching a useful prefix from the root.

\begin{theorem}
\label{thm:hidden-path-no-reset}
Fix $K\ge 2$, $H\ge 1$, and $\lambda>0$. Let $z,z'\in\Sigma^H$ satisfy $z_{1:H-1}=z'_{1:H-1}$ and $z_H\neq z'_H$. Let $A$ be any randomized adaptive algorithm that makes at most $q$ queries to any no-reset experiment at the fixed prompt. Then
\[
\TV\!\Bigl(\mathcal L\bigl(T_q^A(M_z)\bigr),\mathcal L\bigl(T_q^A(M_{z'})\bigr)\Bigr)\le q\,p_+^{H-1}.
\]
Consequently, any such algorithm that distinguishes $M_z$ from $M_{z'}$ with worst-case success probability at least $2/3$ must satisfy
\[
q\ge \frac{1}{3p_+^{H-1}}.
\]
\end{theorem}

The reason is simple. The two models differ only at the depth-$H-1$ prefix $z_{1:H-1}$. A root-start rollout reaches that prefix only if it follows the hidden path correctly for $H-1$ consecutive steps, which happens with probability $p_+^{H-1}$. The reachability barrier from Section~\ref{sec:no-reset} then applies immediately. For fixed $K$ and $\lambda$, this is exponential in the horizon.

The point of the family is that weak local reset breaks this exponential barrier. The learner does not need arbitrary teleportation to deep nodes. It is enough to build the prefix one token at a time and revisit prefixes it has already constructed. The next algorithm does exactly that with chosen-prefix next-token samples.

\begin{algorithm}[t]
\caption{Recovering a hidden path with chosen-prefix samples}
\label{alg:hidden-path-prefix-sample}
\begin{algorithmic}[1]
\STATE Initialize $\widehat z_{<1}:=\varnothing$
\FOR{$t=1$ to $H$}
    \FOR{$j=1$ to $m$}
        \STATE Query $\mathsf{PrefixSample}(\widehat z_{<t})$ and receive $A_j^{(t)}\in\Sigma$
    \ENDFOR
    \FOR{each $a\in\Sigma$}
        \STATE Set $N_a^{(t)}:=\sum_{j=1}^{m}\one\{A_j^{(t)}=a\}$
    \ENDFOR
    \STATE Set $\widehat z_t\in\arg\max_{a\in\Sigma} N_a^{(t)}$
\ENDFOR
\STATE Return $\widehat z=(\widehat z_1,\dots,\widehat z_H)$
\end{algorithmic}
\end{algorithm}

The algorithm starts at the root and keeps a running reconstruction of the hidden prefix. At stage $t$, it repeatedly queries the next-token distribution at the current reconstructed prefix $\widehat z_{<t}$. If the reconstruction so far is correct, then the true next token $z_t$ appears with probability $p_+$ while every rival appears with probability $p_-$. The empirical majority therefore identifies $z_t$ once $m$ is large enough. The next round moves to the one-token extension that was just recovered. No deep jump is ever made, so the procedure obeys the local-reset discipline exactly as defined in Section~\ref{sec:setup}.

\begin{theorem}
\label{thm:hidden-path-prefix-sample}
Fix $K\ge 2$, $H\ge 1$, $\lambda>0$, and $\delta\in(0,1)$. Let
\[
m:=\left\lceil \frac{2}{\Delta^2}\log\!\left(\frac{H(K-1)}{\delta}\right)\right\rceil.
\]
Then Algorithm~\ref{alg:hidden-path-prefix-sample} recovers the hidden path $z$ with probability at least $1-\delta$ using at most $Hm$ chosen-prefix samples. The queried prefixes obey the local-reset discipline.
\end{theorem}

For fixed $K$ and $\lambda$, the gap $\Delta$ is a positive constant, so the query complexity is
\[
O\!\left(H\log\!\frac{H}{\delta}\right).
\]
On the same family, no-reset interaction requires exponentially many queries on some pair of models, while weak local reset with conditional next-token samples already gives a polynomial reconstruction.

The hidden path is the right family for the first axis, but not for the second. Once control is available, even top-$1$ access is enough on this family, because each informative prefix has a unique best child. To see a genuine observation gap after control has been granted, we need branching.

\section{Observation matters after control}
\label{sec:after-control}

The hidden path shows that prefix control is the first boundary. It says little about what richer observations add once control is available, since there is only one informative continuation at each useful prefix. To isolate the observation boundary, we need a family where the top token is unique at every prefix yet still reveals nothing about the branching structure.

We work with a family of prefix trees in which one continuation is a stable default suggestion, shared across all models, while the genuinely informative branch is encoded in the lower-ranked coordinates. This is still a branching family, but it is branching in a way that does not disappear under a unique top token.

Assume $K\ge 3$, and let token $1$ play the role of a distinguished leader token.

\begin{definition}
\label{def:leader-trie}
A leader trie of depth $H$ is a prefix-closed set $T\subseteq \Sigma^{\le H}$ such that $\varnothing\in T$, every leaf of $T$ has depth exactly $H$, and every internal node $p\in I(T):=T\cap \Sigma^{<H}$ has exactly two children in $T$, namely the leader child $p1$ and one additional child $p\,b_T(p)$ for some $b_T(p)\in\{2,\dots,K\}$.
\end{definition}

Thus every internal node has a common leader child and one hidden child. The tree is still branching, but the identity of the leader is the same everywhere.

Now define four probabilities
\[
\alpha:=\frac{4}{K+4}, \qquad
\beta:=\frac{2}{K+4}, \qquad
\gamma:=\frac{1}{K+4}, \qquad
\gamma_0:=\frac{1}{K+3}.
\]
These satisfy
\[
\alpha>\beta>\gamma_0>\gamma.
\]
The associated generator is chosen so that the leader token is always the unique top token, whether or not the queried prefix is informative.

\begin{definition}
\label{def:leader-trie-model}
Given a leader trie $T$, the generator $M_T$ is defined as follows. If $p\in I(T)$, then
\[
M_T(1\mid p)=\alpha, \qquad
M_T\bigl(b_T(p)\mid p\bigr)=\beta, \qquad
M_T(a\mid p)=\gamma \text{ for } a\notin\{1,b_T(p)\}.
\]
If $p\notin I(T)$, then
\[
M_T(1\mid p)=\frac{4}{K+3}, \qquad
M_T(a\mid p)=\gamma_0 \text{ for } a\neq 1.
\]
\end{definition}

At an internal node, the leader token is uniquely best, but one nonleader token is also elevated. At an off-trie prefix, the leader remains uniquely best and all nonleaders sit at the same lower baseline. The informative part of the model is therefore not the top token. It is the identity of the elevated nonleader.

\begin{theorem}
\label{thm:prefix-top-useless}
Assume $K\ge 3$. For every leader trie $T$ and every queried prefix $p\in\Sigma^{<H}$, the oracle $\mathsf{PrefixTop}(p)$ returns the token $1$. Consequently, for any adaptive algorithm interacting only with $\mathsf{PrefixTop}$, the full transcript law is the same for every leader trie of depth $H$. In particular, no such algorithm can distinguish two different leader tries with success probability greater than $1/2$.
\end{theorem}

The theorem is stronger than the tie-based example because it does not rely on any ambiguity in the maximizer. The top token is unique everywhere. It is simply the wrong statistic.

The hidden branch, however, can be recovered from richer observations. The relevant margins are
\[
\Gamma_{\mathrm{lead}}:=\frac{\beta-\gamma_0}{2}
=\frac{K+2}{2(K+4)(K+3)},
\qquad
\gamma_{\mathrm{lead}}:=\frac{\log \beta-\log \gamma_0}{2}
=\frac{1}{2}\log\!\left(\frac{2(K+3)}{K+4}\right).
\]
The probability margin $\Gamma_{\mathrm{lead}}$ separates the hidden child from every nonhidden nonleader. The log-probability margin $\gamma_{\mathrm{lead}}$ is the corresponding gap in log space.

The first recovery result uses chosen-prefix logits and a breadth-first traversal of the trie.

\begin{algorithm}[t]
\caption{Recovering a leader trie with chosen-prefix logits}
\label{alg:leader-trie-logit}
\begin{algorithmic}[1]
\STATE Initialize $\widehat T:=\{\varnothing\}$ and queue $Q:=(\varnothing)$
\WHILE{$Q$ is nonempty}
    \STATE Pop the first prefix $p$ from $Q$
    \STATE Query $\mathsf{PrefixLogit}_{\xi}(p)$ and receive $\widetilde \ell_p\in\mathbb R^\Sigma$
    \STATE Set $\widehat B(p):=\{a\in\{2,\dots,K\}: \widetilde \ell_p(a)>\log \gamma_0+\gamma_{\mathrm{lead}}\}$
    \IF{$\widehat B(p)$ is a singleton $\{\widehat b(p)\}$}
        \STATE Add $p1$ and $p\,\widehat b(p)$ to $\widehat T$
        \IF{$|p|+1<H$}
            \STATE Append $p1$ and $p\,\widehat b(p)$ to $Q$
        \ENDIF
    \ENDIF
\ENDWHILE
\STATE Return $\widehat T$
\end{algorithmic}
\end{algorithm}

The queue stores prefixes whose children still need to be inspected. At a queried prefix $p$, line $4$ asks for the full next-token score vector. Line $5$ looks only at nonleader coordinates and checks which of them lie above the fixed threshold halfway between $\log \beta$ and $\log \gamma_0$. If exactly one nonleader coordinate clears that threshold, then $p$ is an internal node and the corresponding token is the hidden child. Lines $6$ through $8$ then add both active children, the common leader child and the hidden child, and place them in the queue if they have not yet reached depth $H$. If no nonleader clears the threshold, then $p$ is not an internal node and the algorithm stops expanding it.

\begin{theorem}
\label{thm:leader-trie-logit}
Let $T$ be a leader trie. If $\xi<\gamma_{\mathrm{lead}}$, then Algorithm~\ref{alg:leader-trie-logit} recovers $T$ exactly from $|I(T)|$ chosen-prefix logit queries. The queried prefixes obey the local-reset discipline.
\end{theorem}

The same breadth-first idea works with chosen-prefix samples. At each queried prefix, instead of reading the logit vector once, the learner takes repeated conditional samples and identifies the hidden child from empirical frequencies.

\begin{theorem}
\label{thm:leader-trie-sample}
Let $T$ be a leader trie and write $s:=|I(T)|$. Suppose the learner knows an upper bound $S\ge s$. Fix $\delta\in(0,1)$, and set
\[
m:=\left\lceil \frac{1}{2\Gamma_{\mathrm{lead}}^2}\log\!\left(\frac{2(K-1)S}{\delta}\right)\right\rceil.
\]
Then the local-reset algorithm obtained by replacing each query in Algorithm~\ref{alg:leader-trie-logit} with $m$ calls to $\mathsf{PrefixSample}(p)$ and replacing line $4$ by the empirical test
\[
\widehat B(p):=\{a\in\{2,\dots,K\}: \widehat P_p(a)>\gamma_0+\Gamma_{\mathrm{lead}}\}
\]
recovers $T$ with probability at least $1-\delta$ using at most $Sm$ chosen-prefix samples.
\end{theorem}

Once control is available, the observation model matters. The branching structure is invisible to top-$1$ access but recoverable from conditional samples or logits.

One score-based interface deserves separate mention. Teacher-forced sequence scoring lies on the control-rich side of the taxonomy for the same reason: it evaluates a completion chosen by the learner rather than one generated on policy.

\begin{corollary}
\label{cor:seqscore-hidden-path}
In the hidden-path family, exact $\mathsf{SeqScore}$ recovers the hidden path using $HK$ chosen-completion score queries.
\end{corollary}

At stage $t$, one simply fixes an arbitrary suffix of length $H-t$, compares the scores of the $K$ one-token extensions of the current prefix followed by that suffix, and keeps the best extension. This is just teacher forcing. It is not a literal chosen-prefix oracle, but it belongs on the same side of the control boundary.

The two witness families can now be summarized in one statement.

\begin{corollary}
\label{cor:control-first-observation-second}
Fix $K\ge 3$ and $\lambda>0$. On the hidden-path family, every no-reset generator interface requires exponentially many queries on some pair of models by Theorem~\ref{thm:hidden-path-no-reset}, while chosen-prefix sampling recovers the hidden path with
\[
O\!\left(H\log\!\frac{H}{\delta}\right)
\]
local-reset queries by Theorem~\ref{thm:hidden-path-prefix-sample}. On the leader-trie family, chosen-prefix top-token access is useless by Theorem~\ref{thm:prefix-top-useless}, while chosen-prefix logits and chosen-prefix sampling recover the trie efficiently under the local-reset discipline by Theorems~\ref{thm:leader-trie-logit} and \ref{thm:leader-trie-sample}. Prefix control is therefore the first boundary in autoregressive generator access, and observation richness becomes a second boundary only after control is available.
\end{corollary}

The next section shows that the same distinction changes the complexity of a standard post-training objective. The reward model and the KL regularizer stay fixed. Only the generator-side interface changes.

\section{Implications for KL-regularized post-training}
\label{sec:kl-bridge}

The same boundary that separates the identification regimes also appears in a standard post-training objective. The reward is outcome-only, the objective is KL-regularized, and the policy class is the full family of prompt-conditioned policies; only the generator access changes.

Fix a prompt distribution $\rho$ on $\mathcal X$, and let $x^\star\in\mathcal X$ be a designated hard prompt with mass $\eta:=\rho(x^\star)>0$. Choose integers $D,L\ge 1$ and set $H:=D+L+1$. Let $v\in\Sigma^D$ be a known scaffold, and let $\tau_0,\tau_1\in\Sigma$ be two distinct terminal tokens. The hidden state of the instance consists of a suffix $s\in\Sigma^L$ and a reward bit $b\in\{0,1\}$. Write
\[
z_{s,b}:=v\,s\,\tau_b\in\Sigma^H
\]
for the target completion.

On the hard prompt $x^\star$, the base generator $Q_s$ behaves like the hidden-path family from Section~\ref{sec:hidden-path} along the string $v\,s$: for the first $D+L$ steps it mildly favors the next token that stays on that path, and once the full prefix $v\,s$ has been reached the final next-token distribution is uniform. On every prompt $x\neq x^\star$, the base generator is a fixed easy model, identical across the whole family. The outcome reward is concentrated on the single target completion:
\[
r_{s,b}(x,y):=R\,\one\{x=x^\star,\ y=z_{s,b}\}.
\]
The scaffold $v$ is known, the hidden suffix $s$ determines where the informative subtree lies, and the reward bit $b$ is hidden even after $s$ is known, because the final token is uniform under the base generator.

We evaluate policies with the KL-regularized objective from Section~\ref{sec:setup},
\[
J_\beta(\pi)=
\mathbb E_{x\sim \rho}
\left[
\mathbb E_{y\sim \pi(\cdot\mid x)} r_{s,b}(x,y)
-
\beta \KL\!\left(\pi(\cdot\mid x)\,\|\,Q_s(\cdot\mid x)\right)
\right].
\]
As before, a reward query chooses a prompt and a policy, samples a completion from that policy, and observes the resulting outcome reward.

The exact optimizer on each instance has the standard Gibbs form. On every easy prompt $x\neq x^\star$, the reward is zero and the optimizer agrees with the base generator. At the hard prompt,
\[
\pi^\star_{s,b}(y\mid x^\star)\propto Q_s(y\mid x^\star)\exp\!\bigl(r_{s,b}(x^\star,y)/\beta\bigr).
\]
So once the hidden state $(s,b)$ has been identified, the optimal policy can be written down exactly.

The upper bound is straightforward once chosen-prefix sampling is available.

\begin{theorem}
\label{thm:bridge-upper}
Fix $\delta\in(0,1)$, and let
\[
m:=\left\lceil \frac{2}{\Delta^2}\log\!\left(\frac{L(K-1)}{\delta}\right)\right\rceil.
\]
For the family above, there is an algorithm with chosen-prefix next-token sampling access to the base generator and one outcome-reward query that outputs the exact maximizer of $J_\beta$ with probability at least $1-\delta$. The algorithm uses at most $H+Lm$ generator queries and one reward query.
\end{theorem}

The algorithm first walks down the known scaffold one token at a time so that the queried prefixes obey the local-reset discipline. It then treats the hidden suffix as a length-$L$ hidden path and recovers it with Algorithm~\ref{alg:hidden-path-prefix-sample}. Once $s$ is known, one reward query to the deterministic policy concentrated on $v\,s\,\tau_0$ identifies the reward bit $b$. At that point both the base generator and the rewarded completion are known, so the Gibbs formula gives the exact optimizer.

The lower bound shows that no-reset access leaves the problem exponentially hard on some instance, even though the reward model and objective have not changed.

Let $q_0:=Q_s(z_{s,b}\mid x^\star)=p_+^{D+L}/K$ be the base-model mass of the rewarding completion, and choose
\[
R:=\beta\log\!\left(\frac{4}{q_0}\right).
\]
Also write $N:=2K^L$ for the number of possible target completions $v\,s\,\tau_b$.

\begin{theorem}
\label{thm:bridge-lower}
Assume the reward scale above. Consider any algorithm that may interleave at most $q_g$ no-reset generator queries with at most $q_r<N$ outcome-reward queries and then outputs a prompt-conditioned policy $\widehat \pi$. Then there exists a choice of hidden suffix $s\in\Sigma^L$ and reward bit $b\in\{0,1\}$ such that
\[
\Pr\!\left(
J_\beta(\widehat \pi)\ge J_\beta(\pi^\star_{s,b})-\frac{\eta\beta}{4}
\right)
\le
q_g p_+^D + \frac{q_r}{N} + \frac{4}{N-q_r}.
\]
\end{theorem}

Each term has a clear meaning. The factor $q_g p_+^D$ is the chance that one of the no-reset generator queries even reaches the scaffold and enters the informative part of the tree. If that never happens, then generator interaction reveals nothing about the hidden suffix. The term $q_r/N$ is the chance that an outcome-reward query simply guesses the rewarded completion by brute force. If that never happens either, then the remaining reward queries can do no more than rule out candidates one by one, which leads to the last term $4/(N-q_r)$.

All of the hardness is concentrated on a single prompt $x^\star$, so this is already a prompt-distributed post-training lower bound rather than a separate single-prompt phenomenon. The average-regret threshold scales with the prompt mass $\eta$.

The cleanest summary comes from keeping the number of reward queries constant.

\begin{corollary}
\label{cor:constant-reward-gap}
Fix $K\ge 2$, $\lambda>0$, $\beta>0$, and any constant $q_r\ge 1$. There exists a constant $c=c(K,\lambda)>1$ such that for every sufficiently large horizon $H$ and every prompt distribution $\rho$ with a designated prompt $x^\star$ of mass $\eta>0$, there is a family of KL-regularized outcome-reward post-training instances of horizon $H$ with the following property. With chosen-prefix next-token sampling access, one reward query and
\[
O\!\left(H\log\!\frac{H}{\delta}\right)
\]
generator queries suffice to output the exact optimizer with probability at least $1-\delta$. With only no-reset generator access, any algorithm that uses at most $q_r$ reward queries and succeeds with probability at least $2/3$ in outputting a policy of average regret at most $\eta\beta/4$ must make
\[
\Omega(c^H)
\]
generator queries.
\end{corollary}

The reward, KL objective, and policy class are identical across the two access regimes. Changing only the generator interface turns a polynomial problem into an exponentially hard one.

\section{Related work}
\label{sec:related}

Recent theory already clarifies two adjacent parts of the picture. Under fixed sampling-style access, coverage of good responses under the base model governs the complexity of exploration, inference-time alignment, and sharpening-style self-improvement \cite{FosterMhammediRohatgi2025,HuangEtAl2025BestOfN,ChenEtAl2025Coverage,HuangEtAl2024Sharpening}. Our no-reset theorem fits naturally into that line of work. It identifies the generator-side reason those coverage quantities appear there: once interaction is confined to root-start rollouts, all trajectory-local APIs collapse to a single experiment whose power is controlled by on-policy reachability.

A second line of work studies stronger access models in RL and related sequential settings. Reset-style interaction changes what can be learned or planned efficiently in reinforcement learning \cite{MhammediFosterRakhlin2024,RohatgiFoster2025}, while conditional queries can make otherwise hard sequential distributions learnable \cite{MahajanKakadeKrishnamurthyZhang2023,LiuMoitra2024}. Our prefix-tree formalism is the autoregressive analogue of that distinction. For generators, the first distinction is not about the richness of observations but about whether the learner can return to a chosen prefix at all.

The paper is also complementary to work on richer feedback and richer intermediate signals. Mousavi-Hosseini and Erdogdu show that process rewards can avoid outcome-reward barriers under fixed generator access \cite{MousaviHosseiniErdogdu2026}. Botta et al., Amani et al., and Tuyls et al. study verifier guidance, partial rationale exposure, and hidden-state-based exploration \cite{BottaLiMehtaAshZhangRisteski2025,AmaniLotfiBaldwinBengioFarajtabarAbbeWest2025,TuylsFosterKrishnamurthyAsh2025}. Our KL theorem keeps the reward model fixed and varies only generator access, while the structural results operate on the distributional generator side, without hidden states or external verifiers.

Recent empirical work on reasoning post-training points to the same tension between support, exploration, and intermediate control \cite{YueEtAl2025RLVR,KaranDu2025Sampling,TanEtAl2026LED}. The present paper does not explain all of those phenomena, but it gives one reason they should be expected: root-start interaction and reset-like prefix control are genuinely different computational regimes.

\section{Conclusion}
\label{sec:conclusion}

The main lesson is that generator access is part of the computational problem. If interaction is confined to root-start rollouts, then output sampling, generated-token log probabilities, top-$k$ reports, and full visited-prefix distributions all collapse to one experiment governed by on-policy reachability. Weak local reset crosses that boundary, and once across it, richer observations become genuinely useful. That distinction is already strong enough to create an exponential gap for KL-regularized outcome-reward post-training.

\bibliographystyle{plain}
\bibliography{refs}


\appendix

\section{Proofs for Section~\ref{sec:no-reset}}
\label{app:no-reset}

The vocabulary and horizon are finite, so all path spaces here are finite as well. The definitions are nonetheless phrased with measurable kernels to preserve the generality of the main-text statements.

\subsection{The canonical no-reset experiment}

Fix a prompt $x\in\mathcal X$. Let $(\mathcal Q,\mathscr Q)$ be a standard Borel query space and $(\mathcal Y,\mathscr Y)$ a standard Borel reply space.

\begin{definition}
\label{def:app-no-reset}
A generator experiment at prompt $x$ is no-reset if there exists a measurable probability kernel
\[
\Gamma_x\bigl(q,y,\mu_1,\dots,\mu_H;\cdot\bigr)
\]
from $\mathcal Q\times \Sigma^H\times \Delta(\Sigma)^H$ to $(\mathcal Y,\mathscr Y)$ such that one query is answered by first drawing $Y\sim M(\cdot\mid x)$, then setting $\mu_t:=M(\cdot\mid x,Y_{<t})$ for $t=1,\dots,H$, and finally drawing the reply from $\Gamma_x(q,Y,\mu_1,\dots,\mu_H;\cdot)$.
\end{definition}

\begin{definition}
\label{def:app-pathfull}
For the fixed prompt $x$, let $\mathcal W:=\Sigma^H\times \Delta(\Sigma)^H$. The experiment $\mathsf{PathFull}_x$ ignores its query parameter, draws
\[
Y\sim M(\cdot\mid x), \qquad \mu_t:=M(\cdot\mid x,Y_{<t}) \text{ for } t=1,\dots,H,
\]
and returns
\[
W_x:=(Y,\mu_1,\dots,\mu_H)\in\mathcal W.
\]
\end{definition}

\begin{definition}
\label{def:app-postprocessing}
Let $\mathcal E$ be an experiment with query space $(\mathcal Q,\mathscr Q)$ and reply space $(\mathcal Z,\mathscr Z)$. A randomized post-processing of $\mathcal E$ is specified by a measurable kernel $\Pi(q,z;\cdot)$ from $\mathcal Q\times\mathcal Z$ to another reply space $(\mathcal Y,\mathscr Y)$. On query $q$, the post-processed experiment first draws the reply $Z$ of $\mathcal E$ and then draws its final reply from $\Pi(q,Z;\cdot)$.
\end{definition}

\begin{theorem}
\label{thm:app-no-reset-collapse}
Fix a prompt $x$. An experiment at prompt $x$ is no-reset if and only if it is a randomized post-processing of $\mathsf{PathFull}_x$.
\end{theorem}

\begin{proof}
Assume first that the experiment $\mathcal E$ is no-reset. By Definition~\ref{def:app-no-reset}, there is a kernel
\[
\Gamma_x\bigl(q,y,\mu_1,\dots,\mu_H;\cdot\bigr)
\]
such that one query to $\mathcal E$ is answered by drawing
\[
W_x=(Y,\mu_1,\dots,\mu_H)\sim \mathsf{PathFull}_x
\]
and then drawing the final reply from that kernel. If we define
\[
\Pi\bigl(q,(y,\mu_1,\dots,\mu_H);\cdot\bigr):=\Gamma_x\bigl(q,y,\mu_1,\dots,\mu_H;\cdot\bigr),
\]
then $\mathcal E$ is a randomized post-processing of $\mathsf{PathFull}_x$.

Conversely, suppose $\mathcal E$ is a randomized post-processing of $\mathsf{PathFull}_x$. Then there exists a kernel $\Pi(q,w;\cdot)$ such that one query to $\mathcal E$ is answered by drawing $W_x\sim\mathsf{PathFull}_x$ and then sampling from $\Pi(q,W_x;\cdot)$. Writing $w=(y,\mu_1,\dots,\mu_H)$, this is exactly Definition~\ref{def:app-no-reset} with
\[
\Gamma_x\bigl(q,y,\mu_1,\dots,\mu_H;\cdot\bigr):=\Pi\bigl(q,(y,\mu_1,\dots,\mu_H);\cdot\bigr).
\]
So $\mathcal E$ is no-reset.
\end{proof}

\subsection{Agreement outside a prefix set}

Fix a prefix set $U\subseteq\Sigma^{<H}$. We say that two prompt-conditioned generators $M(\cdot\mid x)$ and $M'(\cdot\mid x)$ agree outside $U$ if
\[
M(\cdot\mid x,p)=M'(\cdot\mid x,p) \qquad \text{for every } p\in\Sigma^{<H}\setminus U.
\]

The first lemma shows that, under this assumption, both generators assign the same probability to ever entering $U$.

\begin{lemma}
\label{lem:app-common-reachability}
If $M(\cdot\mid x)$ and $M'(\cdot\mid x)$ agree outside $U$, then
\[
\Reach_M(x,U)=\Reach_{M'}(x,U).
\]
\end{lemma}

\begin{proof}
For each prefix $p\in U$, let $|p|$ be its length and define the event
\[
F_p:=\Bigl\{Y_{1:|p|}=p \text{ and } Y_{1:s}\notin U \text{ for every } s=0,1,\dots,|p|-1\Bigr\},
\]
with the convention $Y_{1:0}=\varnothing$. The event $F_p$ says that the rollout enters $U$ for the first time at the prefix $p$.

The events $\{F_p:p\in U\}$ are pairwise disjoint, and their union is exactly the reachability event $E_U(x)$. It therefore suffices to show that $\Pr_M(F_p\mid x)=\Pr_{M'}(F_p\mid x)$ for every $p\in U$.

Fix $p\in U$ of length $\ell$. On the event $F_p$, every proper prefix of $p$ lies outside $U$. Since the two generators agree outside $U$,
\[
M(p_{s+1}\mid x,p_{1:s})=M'(p_{s+1}\mid x,p_{1:s}) \qquad \text{for every } s=0,1,\dots,\ell-1.
\]
Multiplying these transition probabilities gives
\[
\Pr_M(F_p\mid x)=\prod_{s=0}^{\ell-1} M(p_{s+1}\mid x,p_{1:s})
=\prod_{s=0}^{\ell-1} M'(p_{s+1}\mid x,p_{1:s})
=\Pr_{M'}(F_p\mid x).
\]

Summing over all $p\in U$ yields
\[
\Reach_M(x,U)=\sum_{p\in U}\Pr_M(F_p\mid x)=\sum_{p\in U}\Pr_{M'}(F_p\mid x)=\Reach_{M'}(x,U).
\]
\end{proof}

\subsection{One-query indistinguishability}

We now show that, for the canonical experiment, the two models can differ only on rollouts that reach $U$.

\begin{lemma}
\label{lem:app-one-query-pathfull}
Let $M(\cdot\mid x)$ and $M'(\cdot\mid x)$ agree outside $U$. Let
\[
W_x=(Y,\mu_1,\dots,\mu_H)
\quad \text{and} \quad
W_x'=(Y',\mu_1',\dots,\mu_H')
\]
denote one reply of $\mathsf{PathFull}_x$ under $M$ and $M'$, respectively. Then
\[
\TV\bigl(\mathcal L(W_x),\mathcal L(W_x')\bigr)\le \Reach_M(x,U).
\]
\end{lemma}

\begin{proof}
Let $E:=E_U(x)$ be the event that the rollout under $M$ reaches $U$, and let $E'$ be the corresponding event under $M'$. By Lemma~\ref{lem:app-common-reachability},
\[
\Pr(E)=\Pr(E')=\Reach_M(x,U).
\]

We claim that the conditional laws of $W_x$ and $W_x'$ agree on the complements of these events:
\[
\mathcal L(W_x\mid E^c)=\mathcal L(W_x'\mid (E')^c).
\]
To see this, fix any trajectory $y\in\Sigma^H$ such that none of its visited prefixes lies in $U$. Along that trajectory, every prefix $y_{<t}$ lies outside $U$, so
\[
M(y_t\mid x,y_{<t})=M'(y_t\mid x,y_{<t}) \qquad \text{for every } t=1,\dots,H.
\]
Therefore the trajectory $y$ has the same probability under both models. The same observation gives
\[
M(\cdot\mid x,y_{<t})=M'(\cdot\mid x,y_{<t}) \qquad \text{for every } t,
\]
so the full reply $(y,\mu_1,\dots,\mu_H)$ is identical in the two worlds. Summing over all trajectories that avoid $U$ proves the claim.

Write
\[
R:=\mathcal L(W_x\mid E^c)=\mathcal L(W_x'\mid (E')^c),
\]
and define
\[
P_1:=\mathcal L(W_x\mid E), \qquad Q_1:=\mathcal L(W_x'\mid E').
\]
Then
\[
\mathcal L(W_x)=\Reach_M(x,U)\,P_1+\bigl(1-\Reach_M(x,U)\bigr)R
\]
and
\[
\mathcal L(W_x')=\Reach_M(x,U)\,Q_1+\bigl(1-\Reach_M(x,U)\bigr)R.
\]
Subtracting the two measures gives
\[
\mathcal L(W_x)-\mathcal L(W_x')=\Reach_M(x,U)\,(P_1-Q_1).
\]
Taking total variation and using $\TV(P_1,Q_1)\le 1$ yields
\[
\TV\bigl(\mathcal L(W_x),\mathcal L(W_x')\bigr)\le \Reach_M(x,U).
\]
\end{proof}

The same one-query bound holds for every no-reset experiment by data processing.

\begin{corollary}
\label{cor:app-one-query-no-reset}
Let $\mathcal E$ be any no-reset experiment at prompt $x$. If the two generators agree outside $U$, then for every fixed query parameter $q$ the reply laws under $M$ and $M'$ satisfy
\[
\TV\bigl(\mathcal L(X_q),\mathcal L(X_q')\bigr)\le \Reach_M(x,U).
\]
\end{corollary}

\begin{proof}
By Theorem~\ref{thm:app-no-reset-collapse}, the experiment $\mathcal E$ is a randomized post-processing of $\mathsf{PathFull}_x$. Therefore the reply $X_q$ is obtained by applying the same Markov kernel to $W_x$ that produces $X_q'$ from $W_x'$. Total variation cannot increase under a common Markov kernel, so Lemma~\ref{lem:app-one-query-pathfull} gives
\[
\TV\bigl(\mathcal L(X_q),\mathcal L(X_q')\bigr)\le
\TV\bigl(\mathcal L(W_x),\mathcal L(W_x')\bigr)\le \Reach_M(x,U).
\]
\end{proof}

\subsection{Adaptive interaction}

We now pass from one query to a fully adaptive transcript.

\begin{definition}
\label{def:app-transcript}
Fix a query space $(\mathcal Q,\mathscr Q)$ and a reply space $(\mathcal Y,\mathscr Y)$. A randomized adaptive algorithm that makes at most $q$ queries consists of a random seed together with measurable decision rules that, at each round, either stop or submit a new query based on the previous transcript. If the algorithm stops early, the remaining rounds are padded with a distinguished symbol $\bot$. The padded transcript is written as
\[
T_q^A(M,x)\in\bigl((\mathcal Q\times\mathcal Y)\cup\{\bot\}\bigr)^q.
\]
\end{definition}

\begin{theorem}
\label{thm:app-reachability-barrier}
Let $\mathcal E$ be any no-reset experiment at prompt $x$. Suppose $M(\cdot\mid x)$ and $M'(\cdot\mid x)$ agree outside $U$. Then every randomized adaptive algorithm $A$ that makes at most $q$ queries satisfies
\[
\TV\!\Bigl(\mathcal L\bigl(T_q^A(M,x)\bigr),\mathcal L\bigl(T_q^A(M',x)\bigr)\Bigr)\le q\,\Reach_M(x,U).
\]
\end{theorem}

\begin{proof}
We first treat the case in which $A$ is deterministic. We construct a coupling of the two transcript laws round by round.

Let $Z_r$ and $Z_r'$ be the $r$th padded transcript entries under $M$ and $M'$. At round $1$, the deterministic algorithm either stops immediately or chooses the same first query in both worlds. If it stops immediately, the two transcripts are identical and there is nothing to prove. Otherwise, Corollary~\ref{cor:app-one-query-no-reset} shows that the reply laws to this common query differ in total variation by at most $\Reach_M(x,U)$. By the coupling characterization of total variation, the two first replies can therefore be coupled so that
\[
\Pr(Z_1\neq Z_1')\le \Reach_M(x,U).
\]

Now suppose the coupling has been defined through round $r-1$. On the event that the first $r-1$ transcript entries agree, the deterministic algorithm is in the same state in the two worlds, so it either stops in both or chooses the same round-$r$ query in both. If it stops, we couple the remaining padded symbols identically. If it continues, Corollary~\ref{cor:app-one-query-no-reset} again shows that the two reply laws to the common query differ in total variation by at most $\Reach_M(x,U)$, so we can choose the conditional coupling so that
\[
\Pr\bigl(Z_r\neq Z_r' \mid Z_1=Z_1',\dots,Z_{r-1}=Z_{r-1}'\bigr)\le \Reach_M(x,U).
\]

Let
\[
D_r:=\{(Z_1,\dots,Z_r)\neq (Z_1',\dots,Z_r')\}
\]
be the event that the two transcripts have diverged by round $r$, and set $D_0:=\varnothing$. Then
\[
D_r\setminus D_{r-1}\subseteq
\{Z_1=Z_1',\dots,Z_{r-1}=Z_{r-1}',\, Z_r\neq Z_r'\}.
\]
Taking probabilities and using the conditional bound above gives
\[
\Pr(D_r\setminus D_{r-1})\le \Reach_M(x,U) \qquad \text{for every } r=1,\dots,q.
\]
Summing over $r$ yields
\[
\Pr(D_q)\le q\,\Reach_M(x,U).
\]
By the coupling characterization of total variation,
\[
\TV\!\Bigl(\mathcal L\bigl(T_q^A(M,x)\bigr),\mathcal L\bigl(T_q^A(M',x)\bigr)\Bigr)\le \Pr(D_q)\le q\,\Reach_M(x,U).
\]

It remains to handle randomized algorithms. Let $R$ denote the algorithm's random seed. Conditional on $R=r$, the algorithm becomes deterministic, so the bound above gives
\[
\TV\!\Bigl(\mathcal L\bigl(T_q^A(M,x)\mid R=r\bigr),\mathcal L\bigl(T_q^A(M',x)\mid R=r\bigr)\Bigr)\le q\,\Reach_M(x,U)
\]
for every seed value $r$. The unconditional transcript laws are mixtures of these conditional laws with the same mixing distribution, because the seed is sampled independently of the model. Convexity of total variation under common mixing then gives the same inequality without conditioning on $R$.
\end{proof}

The testing lower bound used later is an immediate corollary.

\begin{corollary}
\label{cor:app-testing}
Suppose an algorithm makes at most $q$ adaptive queries to a no-reset experiment at prompt $x$ and then outputs a guess in $\{M,M'\}$. If its worst-case success probability is at least $2/3$, then
\[
q\ge \frac{1}{3\,\Reach_M(x,U)}.
\]
\end{corollary}

\begin{proof}
Let
\[
P:=\mathcal L\bigl(T_q^A(M,x)\bigr), \qquad Q:=\mathcal L\bigl(T_q^A(M',x)\bigr).
\]
For any decision rule based on the transcript, the average success probability under the uniform prior on $\{M,M'\}$ is at most
\[
\frac{1}{2}+\frac{1}{2}\TV(P,Q).
\]
By Theorem~\ref{thm:app-reachability-barrier},
\[
\TV(P,Q)\le q\,\Reach_M(x,U).
\]
If the worst-case success probability is at least $2/3$, then the average success probability under the uniform prior is also at least $2/3$. Hence
\[
\frac{1}{2}+\frac{q\,\Reach_M(x,U)}{2}\ge \frac{2}{3},
\]
which rearranges to the claimed lower bound.
\end{proof}

\section{Proofs for Section~\ref{sec:hidden-path}}
\label{app:hidden-path}

\subsection{No-reset hardness on hidden paths}

Recall the hidden-path family from Definition~\ref{def:hidden-path}. For convenience, we repeat the transition rule:
\[
M_z(a\mid p)=
\begin{cases}
p_+, & \text{if } p=z_{<t} \text{ for some } t\in\{1,\dots,H\} \text{ and } a=z_t,\\
p_-, & \text{if } p=z_{<t} \text{ for some } t\in\{1,\dots,H\} \text{ and } a\neq z_t,\\
1/K, & \text{if } p \text{ is not a prefix of } z.
\end{cases}
\]

\begin{theorem}
\label{thm:app-hidden-path-no-reset}
Fix $K\ge 2$, $H\ge 1$, and $\lambda>0$. Let $z,z'\in\Sigma^H$ satisfy $z_{1:H-1}=z'_{1:H-1}$ and $z_H\neq z'_H$. Let $A$ be any randomized adaptive algorithm that makes at most $q$ queries to any no-reset generator experiment. Then
\[
\TV\!\Bigl(\mathcal L\bigl(T_q^A(M_z,x)\bigr),\mathcal L\bigl(T_q^A(M_{z'},x)\bigr)\Bigr)\le q\,p_+^{H-1}.
\]
Consequently, any such algorithm that distinguishes $M_z$ from $M_{z'}$ with worst-case success probability at least $2/3$ must satisfy
\[
q\ge \frac{1}{3p_+^{H-1}}.
\]
\end{theorem}

\begin{proof}
Let
\[
c:=z_{1:H-1}=z'_{1:H-1}.
\]
We first show that $M_z$ and $M_{z'}$ agree outside the singleton set
\[
U:=\{c\}.
\]

Fix any prefix $p\in\Sigma^{<H}$ with $p\neq c$. If $p$ is not a prefix of $z$, then because $z$ and $z'$ agree through depth $H-1$, it is also not a prefix of $z'$. In that case both models are uniform at $p$. If instead $p$ is a prefix of $z$, then since $p\neq c$ and $z$ and $z'$ agree on their first $H-1$ symbols, the same prefix $p$ is also a prefix of $z'$, and the favored next token is the same under both models. Thus
\[
M_z(\cdot\mid p)=M_{z'}(\cdot\mid p)
\qquad\text{for every } p\neq c.
\]
So the two generators agree outside $U$.

We next compute the reachability of $U$. A rollout reaches the prefix $c$ if and only if its first $H-1$ tokens are exactly $c$. Under either model, each of those transitions follows the hidden path and therefore has conditional probability $p_+$. Hence
\[
\Reach_{M_z}(x,U)=\Reach_{M_{z'}}(x,U)=p_+^{H-1}.
\]

Theorem~\ref{thm:app-reachability-barrier} from Appendix~\ref{app:no-reset} now gives
\[
\TV\!\Bigl(\mathcal L\bigl(T_q^A(M_z,x)\bigr),\mathcal L\bigl(T_q^A(M_{z'},x)\bigr)\Bigr)\le q\,p_+^{H-1}.
\]
The testing lower bound then follows from Corollary~\ref{cor:app-testing}.
\end{proof}

\subsection{Recovery from chosen-prefix samples}

We now prove the upper bound for Algorithm~\ref{alg:hidden-path-prefix-sample}.

\begin{theorem}
\label{thm:app-hidden-path-prefix-sample}
Fix $K\ge 2$, $H\ge 1$, $\lambda>0$, and $\delta\in(0,1)$. Let
\[
m:=\left\lceil \frac{2}{\Delta^2}\log\!\left(\frac{H(K-1)}{\delta}\right)\right\rceil.
\]
Then Algorithm~\ref{alg:hidden-path-prefix-sample} recovers the hidden path $z$ with probability at least $1-\delta$ using at most $Hm$ chosen-prefix samples. The queried prefixes obey the local-reset discipline.
\end{theorem}

\begin{proof}
For each stage $t\in\{1,\dots,H\}$, let
\[
E_t:=\{\widehat z_{<t}=z_{<t}\}
\]
be the event that the algorithm has reconstructed the hidden path correctly through depth $t-1$.

Fix a stage $t$ and condition on $E_t$. Then the query prefix at stage $t$ is the true hidden prefix $z_{<t}$. Therefore the replies
\[
A_1^{(t)},\dots,A_m^{(t)}
\]
are independent samples from the next-token distribution at $z_{<t}$. In particular,
\[
\Pr\bigl(A_j^{(t)}=z_t \mid E_t\bigr)=p_+
\qquad\text{and}\qquad
\Pr\bigl(A_j^{(t)}=a \mid E_t\bigr)=p_-
\quad\text{for every } a\neq z_t.
\]

Fix any competing token $a\neq z_t$. For each sample index $j$, define
\[
X_j^{(a,t)}:=\one\{A_j^{(t)}=z_t\}-\one\{A_j^{(t)}=a\}.
\]
Conditional on $E_t$, the variables $X_1^{(a,t)},\dots,X_m^{(a,t)}$ are independent, each lies in the interval $[-1,1]$, and
\[
\mathbb E\bigl[X_j^{(a,t)}\mid E_t\bigr]=p_+-p_-=\Delta.
\]

By definition of the empirical counts,
\[
N_{z_t}^{(t)}-N_a^{(t)}=\sum_{j=1}^{m} X_j^{(a,t)}.
\]
Therefore
\[
\Pr\bigl(N_a^{(t)}\ge N_{z_t}^{(t)} \mid E_t\bigr)
=
\Pr\!\left(\sum_{j=1}^{m} X_j^{(a,t)}\le 0 \,\middle|\, E_t\right).
\]
Subtracting the mean from each summand gives
\[
\Pr\!\left(\sum_{j=1}^{m}\bigl(X_j^{(a,t)}-\Delta\bigr)\le -m\Delta \,\middle|\, E_t\right).
\]
Hoeffding's inequality for independent variables in $[-1,1]$ yields
\[
\Pr\bigl(N_a^{(t)}\ge N_{z_t}^{(t)} \mid E_t\bigr)\le \exp\!\left(-\frac{m\Delta^2}{2}\right).
\]

There are $K-1$ competitors. A union bound therefore gives
\[
\Pr\bigl(\widehat z_t\neq z_t \mid E_t\bigr)\le (K-1)\exp\!\left(-\frac{m\Delta^2}{2}\right)\le \frac{\delta}{H},
\]
where the last inequality follows from the choice of $m$.

We now propagate the stagewise guarantee. Since
\[
E_{t+1}^c=E_t^c\cup\bigl(E_t\cap\{\widehat z_t\neq z_t\}\bigr),
\]
we have
\[
\Pr(E_{t+1}^c)\le \Pr(E_t^c)+\Pr\bigl(\widehat z_t\neq z_t \mid E_t\bigr)\Pr(E_t)
\le \Pr(E_t^c)+\frac{\delta}{H}.
\]
Because $E_1$ is the sure event, $\Pr(E_1^c)=0$. Iterating the inequality from $t=1$ to $t=H$ gives
\[
\Pr(E_{H+1}^c)\le \delta.
\]
But $E_{H+1}$ is exactly the event $\{\widehat z=z\}$. Hence
\[
\Pr(\widehat z=z)\ge 1-\delta.
\]

The algorithm makes exactly $m$ chosen-prefix queries at each of the $H$ stages, so the total number of queries is $Hm$.

Finally, the local-reset discipline holds because the first queried prefix is $\varnothing$, every repeated query at stage $t$ revisits the same already queried prefix $\widehat z_{<t}$, and the first new query at stage $t+1$ is the one-token extension $\widehat z_{<t}\widehat z_t$ of a previously queried prefix.
\end{proof}

\subsection{Exact teacher-forced sequence scores}

We also prove the sequence-scoring corollary from the main text.

\begin{corollary}
\label{cor:app-seqscore-hidden-path}
In the hidden-path family, exact $\mathsf{SeqScore}$ recovers the hidden path using $HK$ chosen-completion score queries.
\end{corollary}

\begin{proof}
Fix any rule that assigns an arbitrary suffix $s^{(t)}\in\Sigma^{H-t}$ to each stage $t\in\{1,\dots,H\}$. The recovery procedure maintains a current prefix $\widehat z_{<t}$ and, at stage $t$, queries the exact scores of the $K$ completions
\[
\widehat z_{<t}\,a\,s^{(t)}, \qquad a\in\Sigma.
\]
It then sets $\widehat z_t$ to the token with the largest score.

We prove by induction on $t$ that $\widehat z_{<t}=z_{<t}$ for every stage. The base case $t=1$ is immediate. Assume now that $\widehat z_{<t}=z_{<t}$. Let $c:=z_{<t}$, and fix any incorrect token $a\neq z_t$. We compare the probabilities of the completions
\[
y_{z_t}:=c\,z_t\,s^{(t)}
\qquad\text{and}\qquad
y_a:=c\,a\,s^{(t)}.
\]

The completion $y_a$ leaves the hidden path immediately at time $t$, so every later transition is uniform. Hence
\[
\Pr_{M_z}(Y=y_a\mid x)
=
\Pr_{M_z}(Y_{<t}=c\mid x)\cdot p_- \cdot K^{-(H-t)}.
\]

The completion $y_{z_t}$ stays on the hidden path for at least one additional step. Suppose it remains on the hidden path for exactly $r+1$ consecutive steps after time $t$, where $r$ may be any value in $\{0,\dots,H-t\}$. Then
\[
\Pr_{M_z}(Y=y_{z_t}\mid x)
=
\Pr_{M_z}(Y_{<t}=c\mid x)\cdot p_+^{r+1}\cdot u,
\]
where $u$ is either $p_-K^{-(H-t-r-1)}$ if the suffix eventually leaves the hidden path, or $1$ if the suffix coincides with the true remaining suffix all the way to the end. In either case,
\[
u\ge p_-K^{-(H-t-r)}.
\]
Therefore
\[
\frac{\Pr_{M_z}(Y=y_{z_t}\mid x)}{\Pr_{M_z}(Y=y_a\mid x)}
\ge (Kp_+)^{r+1}.
\]
Since
\[
Kp_+=\frac{Ke^\lambda}{e^\lambda+K-1}>1,
\]
this ratio is strictly larger than $1$. Thus $y_{z_t}$ has strictly larger probability than every $y_a$ with $a\neq z_t$.

Because $\mathsf{SeqScore}$ returns the exact log probability of the queried completion, the correct token $z_t$ is the unique maximizer at stage $t$. Hence $\widehat z_t=z_t$, which completes the induction.

There are $K$ score queries at each of the $H$ stages, so the total number of queries is $HK$.
\end{proof}

\section{Proofs for Section~\ref{sec:after-control}}
\label{app:after-control}

\subsection{Basic properties of the leader-trie family}

Recall the leader-trie model from Definitions~\ref{def:leader-trie} and \ref{def:leader-trie-model}. The useful inequalities are
\[
\alpha=\frac{4}{K+4}, \qquad
\beta=\frac{2}{K+4}, \qquad
\gamma=\frac{1}{K+4}, \qquad
\gamma_0=\frac{1}{K+3},
\]
so
\[
\alpha>\beta>\gamma_0>\gamma.
\]
The relevant margins are
\[
\Gamma_{\mathrm{lead}}=\frac{\beta-\gamma_0}{2}=\frac{K+2}{2(K+4)(K+3)}
\]
and
\[
\gamma_{\mathrm{lead}}=\frac{\log\beta-\log\gamma_0}{2}
=\frac{1}{2}\log\!\left(\frac{2(K+3)}{K+4}\right)>0.
\]

\subsection{Top-token access is useless}

\begin{theorem}
\label{thm:app-prefix-top-useless}
Assume $K\ge 3$. For every leader trie $T$ and every queried prefix $p\in\Sigma^{<H}$, the oracle $\mathsf{PrefixTop}(p)$ returns the token $1$. Consequently, for any adaptive algorithm interacting only with $\mathsf{PrefixTop}$, the full transcript law is the same for every leader trie of depth $H$. In particular, no such algorithm can distinguish two different leader tries with success probability greater than $1/2$.
\end{theorem}

\begin{proof}
Fix a leader trie $T$ and a queried prefix $p\in\Sigma^{<H}$.

If $p\in I(T)$, then by Definition~\ref{def:leader-trie-model},
\[
M_T(1\mid p)=\alpha, \qquad M_T\bigl(b_T(p)\mid p\bigr)=\beta, \qquad M_T(a\mid p)=\gamma \text{ for } a\notin\{1,b_T(p)\}.
\]
Since $\alpha>\beta>\gamma$, the unique maximizer is the token $1$.

If instead $p\notin I(T)$, then
\[
M_T(1\mid p)=\frac{4}{K+3}, \qquad M_T(a\mid p)=\gamma_0 \text{ for } a\neq 1.
\]
Again the unique maximizer is the token $1$.

So every query to $\mathsf{PrefixTop}$ returns the same deterministic reply, namely $1$, regardless of the trie and regardless of the queried prefix.

Now let $A$ be any adaptive algorithm. At every round, the distribution of the next query depends only on the previous transcript and the algorithm's internal randomness. Since the reply is always the constant token $1$, the transcript distribution is the same for every leader trie. Therefore any decision rule based on the transcript receives identically distributed data under any two candidate tries. Under the uniform prior on a pair of different tries, no such rule can have average success probability above $1/2$, and hence no rule can have worst-case success probability above $1/2$ either.
\end{proof}

\subsection{Recovery from chosen-prefix logits}

We now prove that the hidden branch can be recovered from chosen-prefix logit access.

\begin{theorem}
\label{thm:app-leader-trie-logit}
Let $T$ be a leader trie. If $\xi<\gamma_{\mathrm{lead}}$, then Algorithm~\ref{alg:leader-trie-logit} recovers $T$ exactly from $|I(T)|$ chosen-prefix logit queries. The queried prefixes obey the local-reset discipline.
\end{theorem}

\begin{proof}
We first show that, at every queried prefix, the thresholding rule in line $4$ of Algorithm~\ref{alg:leader-trie-logit} behaves exactly as intended.

Fix a prefix $p\in\Sigma^{<H}$.

Suppose first that $p\in I(T)$. Then the hidden child is $b_T(p)\in\{2,\dots,K\}$. For that token,
\[
\log M_T\bigl(b_T(p)\mid p\bigr)=\log\beta.
\]
Since the returned vector satisfies
\[
\bigl\|\widetilde\ell_p-\log M_T(\cdot\mid p)\bigr\|_\infty\le \xi,
\]
we have
\[
\widetilde\ell_p\bigl(b_T(p)\bigr)\ge \log\beta-\xi.
\]
Because $\xi<\gamma_{\mathrm{lead}}$ and $\gamma_{\mathrm{lead}}=(\log\beta-\log\gamma_0)/2$,
\[
\log\beta-\xi > \log\beta-\gamma_{\mathrm{lead}}=\log\gamma_0+\gamma_{\mathrm{lead}}.
\]
So the hidden child belongs to $\widehat B(p)$.

Now fix any other nonleader token $a\in\{2,\dots,K\}\setminus\{b_T(p)\}$. Its true log probability is $\log\gamma$, so
\[
\widetilde\ell_p(a)\le \log\gamma+\xi.
\]
Since $\gamma<\gamma_0$ and $\xi<\gamma_{\mathrm{lead}}$,
\[
\widetilde\ell_p(a)<\log\gamma_0+\gamma_{\mathrm{lead}}.
\]
Thus no other nonleader belongs to $\widehat B(p)$, and we conclude that
\[
\widehat B(p)=\{b_T(p)\}
\qquad\text{for every } p\in I(T).
\]

Suppose next that $p\notin I(T)$. Then every nonleader token has true log probability $\log\gamma_0$, so for every $a\in\{2,\dots,K\}$,
\[
\widetilde\ell_p(a)\le \log\gamma_0+\xi<\log\gamma_0+\gamma_{\mathrm{lead}}.
\]
Hence
\[
\widehat B(p)=\varnothing
\qquad\text{for every } p\notin I(T).
\]

We now prove that the queue-based procedure reconstructs the trie exactly. We argue by induction over the order in which prefixes are popped from the queue.

At initialization, the queue contains only $\varnothing$. Since every leaf of a leader trie has depth exactly $H$ and $H\ge 1$, the root is an internal node. So the first queried prefix is correct.

Assume inductively that every prefix previously popped from the queue was a true internal node of $T$, and that whenever such a prefix $p$ was processed, the algorithm recovered the correct singleton $\widehat B(p)=\{b_T(p)\}$ and therefore added exactly the two true children $p1$ and $p\,b_T(p)$ to $\widehat T$ and, if they were not yet at depth $H$, to the queue.

Let $p$ be the next prefix popped from the queue. By construction of the queue, $p$ was added earlier as one of the two children of a previously processed internal node. Hence $p\in T$. Since only prefixes of depth at most $H-1$ are ever appended to the queue, we also have $|p|<H$. Because every node of $T$ of depth strictly less than $H$ is internal, it follows that $p\in I(T)$. Thus every queried prefix is indeed a true internal node.

Conversely, every internal node is eventually queried. This is proved by induction on depth. The root is queried first. Now fix any internal node $p\neq \varnothing$, and let $p^{-}$ be its parent. By prefix-closure, $p^{-}\in T$, and since $|p^{-}|<H$, it is also an internal node. By the induction hypothesis, $p^{-}$ is eventually queried. When it is processed, the algorithm adds both of its true children to the queue, including $p$. Therefore $p$ is eventually queried.

We have shown that the algorithm queries exactly the internal nodes of $T$. At each such node it recovers the correct hidden child, and therefore it adds exactly the correct children. Thus the returned set $\widehat T$ is exactly $T$.

The number of queries is one per internal node, namely $|I(T)|$.

Finally, the queried prefixes obey the local-reset discipline. The first query is $\varnothing$. Every later queried prefix was added to the queue only after its parent had been queried, so when it is first queried it is a one-token extension of a previously queried prefix. This is exactly the local-reset condition from Section~\ref{sec:setup}.
\end{proof}

\subsection{Recovery from chosen-prefix samples}

For completeness we write the sample-based breadth-first procedure explicitly.

\begin{algorithm}[t]
\caption{Recovering a leader trie with chosen-prefix samples}
\label{alg:leader-trie-sample}
\begin{algorithmic}[1]
\STATE Initialize $\widehat T:=\{\varnothing\}$, queue $Q:=(\varnothing)$, and counter $c:=0$
\WHILE{$Q$ is nonempty and $c<S$}
    \STATE Pop the first prefix $p$ from $Q$ and set $c:=c+1$
    \FOR{$j=1$ to $m$}
        \STATE Query $\mathsf{PrefixSample}(p)$ and receive $A_j^{(p)}\in\Sigma$
    \ENDFOR
    \FOR{each $a\in\{2,\dots,K\}$}
        \STATE Set $\widehat P_p(a):=\frac{1}{m}\sum_{j=1}^{m}\one\{A_j^{(p)}=a\}$
    \ENDFOR
    \STATE Set $\widehat B(p):=\{a\in\{2,\dots,K\}: \widehat P_p(a)>\gamma_0+\Gamma_{\mathrm{lead}}\}$
    \IF{$\widehat B(p)$ is a singleton $\{\widehat b(p)\}$}
        \STATE Add $p1$ and $p\,\widehat b(p)$ to $\widehat T$
        \IF{$|p|+1<H$}
            \STATE Append $p1$ and $p\,\widehat b(p)$ to $Q$
        \ENDIF
    \ENDIF
\ENDWHILE
\IF{$Q$ is empty}
    \STATE Return $\widehat T$
\ELSE
    \STATE Return $\bot$
\ENDIF
\end{algorithmic}
\end{algorithm}

\begin{theorem}
\label{thm:app-leader-trie-sample}
Let $T$ be a leader trie and write $s:=|I(T)|$. Suppose the learner knows an upper bound $S\ge s$. Fix $\delta\in(0,1)$, and set
\[
m:=\left\lceil \frac{1}{2\Gamma_{\mathrm{lead}}^2}\log\!\left(\frac{2(K-1)S}{\delta}\right)\right\rceil.
\]
Then Algorithm~\ref{alg:leader-trie-sample} recovers $T$ with probability at least $1-\delta$ using at most $Sm$ chosen-prefix samples. The queried prefixes obey the local-reset discipline.
\end{theorem}

\begin{proof}
Fix once and for all a breadth-first ordering of the true internal nodes of $T$:
\[
I(T)=\{p_1,\dots,p_s\}.
\]
For each $i\in\{0,\dots,s\}$, let $G_i$ be the event that after exactly $i$ processed prefixes, the following three conditions hold simultaneously:
\[
\text{the processed prefixes are exactly } p_1,\dots,p_i,
\]
\[
\widehat T \text{ is correct on all discovered nodes},
\]
and
\[
Q \text{ contains exactly the remaining internal nodes } p_{i+1},\dots,p_s
\text{ in breadth-first order}.
\]
The event $G_0$ holds deterministically at initialization.

Fix $i\in\{1,\dots,s\}$ and condition on $G_{i-1}$. Then the next processed prefix is the true internal node $p_i$.

For each nonleader token $a\in\{2,\dots,K\}$, the empirical frequency
\[
\widehat P_{p_i}(a)=\frac{1}{m}\sum_{j=1}^{m}\one\{A_j^{(p_i)}=a\}
\]
is the average of $m$ independent Bernoulli random variables with mean $M_T(a\mid p_i)$.

There are two cases. If $a=b_T(p_i)$ is the hidden child, then
\[
M_T(a\mid p_i)=\beta.
\]
If $a\neq b_T(p_i)$, then
\[
M_T(a\mid p_i)=\gamma.
\]
In either case, the distance from the mean to the threshold $\gamma_0+\Gamma_{\mathrm{lead}}=(\beta+\gamma_0)/2$ is at least $\Gamma_{\mathrm{lead}}$. Indeed, for the hidden child,
\[
\beta-(\gamma_0+\Gamma_{\mathrm{lead}})=\Gamma_{\mathrm{lead}},
\]
and for every other nonleader,
\[
(\gamma_0+\Gamma_{\mathrm{lead}})-\gamma
>
(\gamma_0+\Gamma_{\mathrm{lead}})-\gamma_0
=
\Gamma_{\mathrm{lead}}.
\]

Hoeffding's inequality therefore gives
\[
\Pr\bigl(\widehat P_{p_i}(b_T(p_i))\le \gamma_0+\Gamma_{\mathrm{lead}} \mid G_{i-1}\bigr)\le \exp(-2m\Gamma_{\mathrm{lead}}^2)
\]
for the hidden child, and
\[
\Pr\bigl(\widehat P_{p_i}(a)> \gamma_0+\Gamma_{\mathrm{lead}} \mid G_{i-1}\bigr)\le \exp(-2m\Gamma_{\mathrm{lead}}^2)
\]
for every nonhidden nonleader $a$.

A union bound over the $K-1$ nonleader tokens yields
\[
\Pr\bigl(\widehat B(p_i)\neq \{b_T(p_i)\}\mid G_{i-1}\bigr)\le (K-1)\exp(-2m\Gamma_{\mathrm{lead}}^2).
\]
By the definition of $m$,
\[
(K-1)\exp(-2m\Gamma_{\mathrm{lead}}^2)\le \frac{\delta}{2S}.
\]
Thus
\[
\Pr\bigl(\widehat B(p_i)\neq \{b_T(p_i)\}\mid G_{i-1}\bigr)\le \frac{\delta}{2S}.
\]

Whenever the good event $\widehat B(p_i)=\{b_T(p_i)\}$ occurs, the algorithm processes $p_i$ correctly and preserves the breadth-first invariant. Therefore
\[
\Pr(G_i^c\mid G_{i-1})\le \frac{\delta}{2S}.
\]

We now propagate the invariant:
\[
G_i^c=G_{i-1}^c\cup\bigl(G_{i-1}\cap G_i^c\bigr),
\]
so
\[
\Pr(G_i^c)\le \Pr(G_{i-1}^c)+\Pr(G_i^c\mid G_{i-1})\Pr(G_{i-1})
\le \Pr(G_{i-1}^c)+\frac{\delta}{2S}.
\]
Starting from $\Pr(G_0^c)=0$ and iterating gives
\[
\Pr(G_s^c)\le s\cdot \frac{\delta}{2S}\le \frac{\delta}{2}.
\]

On the event $G_s$, the algorithm has processed exactly the internal nodes of $T$, reconstructed the trie correctly, and emptied the queue after at most $s\le S$ iterations. In that case the algorithm returns $\widehat T=T$.

To finish the proof, we note that the bound above already controls all possible classification mistakes on true internal nodes. Since the queue only receives children of correctly processed internal nodes, the event $G_s$ also guarantees that no off-trie prefix is ever processed. Thus
\[
\Pr(\widehat T=T)\ge 1-\delta.
\]
The query count is at most $m$ per processed prefix and at most $S$ processed prefixes, so the total number of chosen-prefix samples is at most $Sm$.

The local-reset discipline holds for the same reason as in the logit-based procedure: the first query is $\varnothing$, and every later queried prefix is a one-token extension of a previously queried prefix.
\end{proof}

\subsection{Summary corollary}

\begin{corollary}
\label{cor:app-control-first-observation-second}
Fix $K\ge 3$ and $\lambda>0$. On the hidden-path family, every no-reset generator interface requires exponentially many queries on some pair of models, while chosen-prefix sampling recovers the hidden path in $O(H\log(H/\delta))$ local-reset queries. On the leader-trie family, chosen-prefix top-token access is useless, while chosen-prefix logits and chosen-prefix sampling recover the trie efficiently under the local-reset discipline. Prefix control is therefore the first boundary, and observation richness becomes a second boundary only after control is available.
\end{corollary}

\begin{proof}
The hidden-path statement is exactly the combination of Theorems~\ref{thm:app-hidden-path-no-reset} and \ref{thm:app-hidden-path-prefix-sample}. For fixed $K$ and $\lambda$, the gap $\Delta$ is a positive constant, so the upper bound
\[
Hm
=
H\left\lceil \frac{2}{\Delta^2}\log\!\left(\frac{H(K-1)}{\delta}\right)\right\rceil
\]
is $O(H\log(H/\delta))$.

The leader-trie statement is exactly the combination of Theorems~\ref{thm:app-prefix-top-useless}, \ref{thm:app-leader-trie-logit}, and \ref{thm:app-leader-trie-sample}. The final sentence is the summary of these two witness-family statements.
\end{proof}

\section{Proofs for Section~\ref{sec:kl-bridge}}
\label{app:kl-bridge}

\subsection{The hard family}

Fix a prompt distribution $\rho$ on $\mathcal X$ and a designated hard prompt $x^\star\in\mathcal X$ with mass $\eta:=\rho(x^\star)>0$. Fix integers $D,L\ge 1$ and set $H:=D+L+1$. Let $v\in\Sigma^D$ be a known scaffold, and let $\tau_0,\tau_1\in\Sigma$ be two distinct terminal tokens.

For each hidden suffix $s=(s_1,\dots,s_L)\in\Sigma^L$, define the hard-prompt base generator $Q_s(\cdot\mid x^\star)$ as follows. Let
\[
u:=v\,s\in\Sigma^{D+L}.
\]
For every prefix $p\in\Sigma^{<H}$, set
\[
Q_s(a\mid x^\star,p)=
\begin{cases}
p_+, & \text{if } p=u_{<t} \text{ for some } t\in\{1,\dots,D+L\} \text{ and } a=u_t,\\
p_-, & \text{if } p=u_{<t} \text{ for some } t\in\{1,\dots,D+L\} \text{ and } a\neq u_t,\\
1/K, & \text{if } |p|=D+L,\\
1/K, & \text{if } p \text{ is not a prefix of } u.
\end{cases}
\]
Thus $Q_s(\cdot\mid x^\star)$ follows the hidden path $u=v\,s$ for its first $D+L$ steps and is uniform at the final step.

For every prompt $x\neq x^\star$, fix once and for all an easy generator $Q_{\mathrm{easy}}(\cdot\mid x)$, the same for all instances. The full prompt-conditioned base generator is
\[
Q_s(\cdot\mid x)=
\begin{cases}
Q_s(\cdot\mid x^\star), & \text{if } x=x^\star,\\
Q_{\mathrm{easy}}(\cdot\mid x), & \text{if } x\neq x^\star.
\end{cases}
\]

Now fix a reward scale $R>0$. For each bit $b\in\{0,1\}$, define the target completion
\[
z_{s,b}:=v\,s\,\tau_b\in\Sigma^H
\]
and the outcome reward
\[
r_{s,b}(x,y):=R\,\one\{x=x^\star,\ y=z_{s,b}\}.
\]
All prompts other than $x^\star$ therefore have zero reward.

The candidate target set at the hard prompt is
\[
\mathcal Z:=\{z_{s,b}: s\in\Sigma^L,\ b\in\{0,1\}\},
\]
whose size is
\[
N:=|\mathcal Z|=2K^L.
\]
The base-model mass of the rewarding completion is
\[
q_0:=Q_s(z_{s,b}\mid x^\star)=\frac{p_+^{D+L}}{K},
\]
which does not depend on either $s$ or $b$.

\subsection{The hard-prompt objective}

For a policy $\pi(\cdot\mid x^\star)$, define the hard-prompt objective
\[
J_{x^\star,\beta}^{s,b}(\pi):=
\mathbb E_{y\sim \pi(\cdot\mid x^\star)} r_{s,b}(x^\star,y)
-
\beta \KL\!\left(\pi(\cdot\mid x^\star)\,\|\,Q_s(\cdot\mid x^\star)\right).
\]

The exact optimizer at the hard prompt has the usual Gibbs form.

\begin{lemma}
\label{lem:gibbs-hard-prompt}
Fix an instance $(s,b)$. Let
\[
Z_{s,b}:=\sum_{y\in\Sigma^H} Q_s(y\mid x^\star)\exp\!\bigl(r_{s,b}(x^\star,y)/\beta\bigr),
\]
and define
\[
\pi^\star_{s,b}(y\mid x^\star):=
\frac{Q_s(y\mid x^\star)\exp\!\bigl(r_{s,b}(x^\star,y)/\beta\bigr)}{Z_{s,b}}.
\]
Then for every policy $\pi(\cdot\mid x^\star)$,
\[
J_{x^\star,\beta}^{s,b}(\pi)
=
\beta\log Z_{s,b}
-
\beta\KL\!\left(\pi(\cdot\mid x^\star)\,\|\,\pi^\star_{s,b}(\cdot\mid x^\star)\right).
\]
In particular, $\pi^\star_{s,b}$ is the unique maximizer of the hard-prompt objective, and the optimal value is $\beta\log Z_{s,b}$.
\end{lemma}

\begin{proof}
Fix any policy $\pi(\cdot\mid x^\star)$. By definition of $\pi^\star_{s,b}$,
\[
\log \pi^\star_{s,b}(y\mid x^\star)
=
\log Q_s(y\mid x^\star)+\frac{r_{s,b}(x^\star,y)}{\beta}-\log Z_{s,b}.
\]
Rearranging,
\[
r_{s,b}(x^\star,y)-\beta\log\!\left(\frac{\pi(y\mid x^\star)}{Q_s(y\mid x^\star)}\right)
=
\beta\log Z_{s,b}
-
\beta\log\!\left(\frac{\pi(y\mid x^\star)}{\pi^\star_{s,b}(y\mid x^\star)}\right).
\]
Taking expectation with respect to $y\sim \pi(\cdot\mid x^\star)$ gives
\[
J_{x^\star,\beta}^{s,b}(\pi)
=
\beta\log Z_{s,b}
-
\beta\KL\!\left(\pi(\cdot\mid x^\star)\,\|\,\pi^\star_{s,b}(\cdot\mid x^\star)\right).
\]
Since KL divergence is always nonnegative and vanishes only when the two distributions agree, the stated optimality properties follow.
\end{proof}

We next show that a policy which places only small mass on the target completion must be far from optimal on the hard prompt once the reward scale has been chosen appropriately.

\begin{lemma}
\label{lem:hard-prompt-gap}
Assume
\[
R=\beta\log\!\left(\frac{4}{q_0}\right).
\]
Let $\pi(\cdot\mid x^\star)$ be any policy, and write
\[
m:=\pi(z_{s,b}\mid x^\star).
\]
If $m\le 1/4$, then
\[
J_{x^\star,\beta}^{s,b}(\pi^\star_{s,b})-J_{x^\star,\beta}^{s,b}(\pi)>\frac{\beta}{4}.
\]
\end{lemma}

\begin{proof}
Let $z:=z_{s,b}$ denote the target completion. By Lemma~\ref{lem:gibbs-hard-prompt},
\[
J_{x^\star,\beta}^{s,b}(\pi^\star_{s,b})=\beta\log Z_{s,b}.
\]
Since the reward is nonzero only at $z$ and equals $R$ there,
\[
Z_{s,b}=1-q_0+q_0 e^{R/\beta}=1-q_0+4=5-q_0.
\]
Hence
\[
J_{x^\star,\beta}^{s,b}(\pi^\star_{s,b})=\beta\log(5-q_0)\ge \beta\log 4=2\beta\log 2.
\]

We now upper bound the value of an arbitrary policy with target mass $m$. By data processing for KL divergence under the map that records whether $y=z$, we have
\[
\KL\!\left(\pi(\cdot\mid x^\star)\,\|\,Q_s(\cdot\mid x^\star)\right)\ge \mathrm{kl}(m\|q_0),
\]
where
\[
\mathrm{kl}(m\|q_0):=m\log\!\left(\frac{m}{q_0}\right)+(1-m)\log\!\left(\frac{1-m}{1-q_0}\right).
\]
Therefore
\[
J_{x^\star,\beta}^{s,b}(\pi)\le Rm-\beta\,\mathrm{kl}(m\|q_0).
\]
Substituting $R=\beta\log(4/q_0)$ gives
\begin{align*}
J_{x^\star,\beta}^{s,b}(\pi)
&\le
\beta\left[
m\log\!\left(\frac{4}{q_0}\right)-m\log\!\left(\frac{m}{q_0}\right)-(1-m)\log\!\left(\frac{1-m}{1-q_0}\right)
\right]\\
&=
\beta\left[
m\log 4-m\log m-(1-m)\log(1-m)+(1-m)\log(1-q_0)
\right]\\
&\le
\beta\bigl[m\log 4+h_2(m)\bigr],
\end{align*}
where $h_2(m):=-m\log m-(1-m)\log(1-m)$ is the binary entropy in nats.

Define
\[
g(m):=m\log 4+h_2(m).
\]
For $m\in(0,1)$,
\[
g'(m)=\log 4+\log\!\left(\frac{1-m}{m}\right)=\log\!\left(\frac{4(1-m)}{m}\right).
\]
Hence $g'(m)>0$ whenever $m<4/5$. In particular, $g$ is increasing on $[0,1/4]$, so
\[
g(m)\le g(1/4)=\log 2+\frac{3}{4}\log\!\left(\frac{4}{3}\right).
\]
Using the elementary bound $\log(1+u)\le u$ with $u=1/3$ gives
\[
\log\!\left(\frac{4}{3}\right)\le \frac{1}{3},
\]
and therefore
\[
g(m)\le \log 2+\frac{1}{4}.
\]
So every policy with $m\le 1/4$ satisfies
\[
J_{x^\star,\beta}^{s,b}(\pi)\le \beta\left(\log 2+\frac{1}{4}\right).
\]

Combining the lower bound on the optimal value with the upper bound above, we obtain
\[
J_{x^\star,\beta}^{s,b}(\pi^\star_{s,b})-J_{x^\star,\beta}^{s,b}(\pi)
\ge
\beta\left(2\log 2-\log 2-\frac{1}{4}\right)
=
\beta\left(\log 2-\frac{1}{4}\right).
\]
Since $\log 2>1/2$, the right-hand side is strictly larger than $\beta/4$.
\end{proof}

The hard-prompt gap lifts immediately to the prompt-distributed objective.

\begin{corollary}
\label{cor:average-gap}
Assume
\[
R=\beta\log\!\left(\frac{4}{q_0}\right).
\]
Let $\pi$ be any prompt-conditioned policy. If
\[
\pi(z_{s,b}\mid x^\star)\le 1/4,
\]
then
\[
J_\beta(\pi^\star_{s,b})-J_\beta(\pi)>\frac{\eta\beta}{4}.
\]
\end{corollary}

\begin{proof}
At every prompt $x\neq x^\star$, the reward is zero. Therefore the contribution of such a prompt to the objective is at most $0$, with equality attained by matching the base generator. The exact optimizer $\pi^\star_{s,b}$ does exactly that on all easy prompts, so
\[
J_\beta(\pi^\star_{s,b})=\eta\,J_{x^\star,\beta}^{s,b}(\pi^\star_{s,b}).
\]
For an arbitrary policy $\pi$,
\[
J_\beta(\pi)\le \eta\,J_{x^\star,\beta}^{s,b}(\pi(\cdot\mid x^\star)).
\]
Applying Lemma~\ref{lem:hard-prompt-gap} to the hard-prompt component and multiplying by $\eta$ gives the claim.
\end{proof}

\subsection{Proof of the upper bound}

\begin{theorem}
\label{thm:app-bridge-upper}
Fix $\delta\in(0,1)$, and let
\[
m:=\left\lceil \frac{2}{\Delta^2}\log\!\left(\frac{L(K-1)}{\delta}\right)\right\rceil.
\]
For the hard family above, there is an algorithm with chosen-prefix next-token sampling access to the base generator and one outcome-reward query that outputs the exact maximizer of $J_\beta$ with probability at least $1-\delta$. The algorithm uses at most $H+Lm$ generator queries and one reward query.
\end{theorem}

\begin{proof}
The algorithm works in three stages.

In the first stage, it walks down the known scaffold $v$ in order to respect the local-reset discipline. It queries the prefixes
\[
\varnothing,\ v_{1:1},\ v_{1:2},\ \dots,\ v_{1:D}.
\]
These are $D+1$ generator queries. Their replies are irrelevant, since the scaffold is known in advance; the only point of the stage is to build the prefix $v$ legally.

In the second stage, it treats the hidden suffix as a length-$L$ hidden-path problem. More precisely, consider any prefix of the form
\[
v\,u, \qquad u\in\Sigma^{t-1}, \qquad t\in\{1,\dots,L\}.
\]
If $u=s_{<t}$, then by construction of $Q_s(\cdot\mid x^\star)$ the next-token distribution at $v\,u$ satisfies
\[
Q_s(s_t\mid x^\star,v\,u)=p_+, \qquad
Q_s(a\mid x^\star,v\,u)=p_- \text{ for every } a\neq s_t.
\]
If $u$ is not a prefix of $s$, then the queried prefix lies off the hidden path and the next-token distribution is uniform. So below the known scaffold $v$, the unknown suffix $s$ behaves exactly like a hidden path of length $L$.

The algorithm therefore applies Algorithm~\ref{alg:hidden-path-prefix-sample} from Section~\ref{sec:hidden-path} starting at the prefix $v$. By Theorem~\ref{thm:app-hidden-path-prefix-sample}, after $Lm$ chosen-prefix sampling queries, it recovers the hidden suffix $s$ with probability at least $1-\delta$.

In the third stage, once $s$ has been recovered, the algorithm identifies the reward bit $b$ with one reward query. It queries the hard prompt $x^\star$ using the deterministic policy concentrated on the completion $v\,s\,\tau_0$. The observed reward equals $R$ if and only if $b=0$, and equals $0$ otherwise. So one reward query determines $b$ exactly.

At that point the hidden state $(s,b)$ is known, and therefore the exact optimizer is known as well. On easy prompts it agrees with the base generator, while at the hard prompt it is given by the Gibbs formula from Lemma~\ref{lem:gibbs-hard-prompt}.

The algorithm fails only if the suffix-recovery stage fails, which occurs with probability at most $\delta$. The total number of generator queries is
\[
(D+1)+Lm \le H+Lm
\]
because $H=D+L+1$. The number of reward queries is one.
\end{proof}

\subsection{A posterior-symmetry lemma for the lower bound}

We now prepare the no-reset lower bound. The hidden state is
\[
\Theta:=(S,B),
\]
where $S$ is uniform on $\Sigma^L$ and $B$ is uniform on $\{0,1\}$, independently. Equivalently, the target completion
\[
Z_\Theta:=z_{S,B}
\]
is uniform on the candidate set $\mathcal Z$ of size $N=2K^L$.

Consider an arbitrary algorithm that may interleave generator and reward queries. Define the event $A$ that at least one no-reset generator query made at the hard prompt $x^\star$ reaches the scaffold prefix $v$ in its internal rollout.

\begin{lemma}
\label{lem:scaffold-hit}
For any algorithm that makes at most $q_g$ no-reset generator queries,
\[
\Pr(A)\le q_g p_+^D.
\]
\end{lemma}

\begin{proof}
Fix one no-reset generator query made at the hard prompt $x^\star$. Its internal rollout reaches the scaffold $v$ if and only if its first $D$ sampled tokens match the scaffold exactly. Under every generator in the family, each of those $D$ transitions has conditional probability $p_+$. Hence the probability that this query reaches $v$ is $p_+^D$.

Generator queries made at prompts other than $x^\star$ are irrelevant for the event $A$. Since there are at most $q_g$ generator queries in total, the union bound gives
\[
\Pr(A)\le q_g p_+^D.
\]
\end{proof}

The next lemma is the key symmetry statement. It handles arbitrary interleaving by tracking only transcripts that remain on the event $A^c$ and have not yet received positive reward.

A safe partial transcript is an interaction history in which no generator query has reached the scaffold and no reward query has yet produced positive reward. For a safe transcript $\tau$, define the surviving candidate set
\[
C(\tau):=
\mathcal Z\setminus
\{y\in\mathcal Z : \text{some reward query at } x^\star \text{ has sampled } y \text{ and received reward }0\}.
\]
Generator queries do not change $C(\tau)$; a negative reward query at the hard prompt removes at most one candidate.

\begin{lemma}
\label{lem:posterior-symmetry}
For every safe partial transcript $\tau$, there exists a nonnegative number $q(\tau)$ such that
\[
\Pr(\Theta=\theta,\tau)=\frac{q(\tau)}{N}
\]
for every hidden state $\theta$ whose target completion lies in $C(\tau)$, and
\[
\Pr(\Theta=\theta,\tau)=0
\]
for every hidden state $\theta$ whose target completion does not lie in $C(\tau)$.

Consequently, conditional on any safe transcript $\tau$, the posterior on $\Theta$ is uniform over the surviving candidate set $C(\tau)$.
\end{lemma}

\begin{proof}
We prove the claim by induction on the length of the safe transcript.

For the empty transcript, take $q(\varnothing)=1$. Since $\Theta$ is uniform over the $N$ possible hidden states, the claim holds.

Now suppose the claim holds for a safe transcript $\tau$, and consider one additional query chosen by the algorithm after observing $\tau$.

If the next query is a generator query, there are two possibilities. Some replies correspond to an internal rollout that reaches the scaffold; those replies do not extend the transcript safely and need not be considered. For a reply $g$ that does extend the transcript safely, the internal rollout has avoided the scaffold. Along such a rollout, every visited prefix lies outside the informative subtree below $v$, and all generators $Q_s(\cdot\mid x^\star)$ agree there. Therefore the conditional probability of observing $g$, given the current safe transcript $\tau$, is the same for every surviving hidden state $\theta$. Call this common probability $p(g\mid \tau)$. Then
\[
\Pr(\Theta=\theta,\tau,g)=\Pr(\Theta=\theta,\tau)\,p(g\mid \tau)=\frac{q(\tau)p(g\mid \tau)}{N}
\]
for every surviving $\theta$, while the nonsurviving states remain at probability zero. So the claim holds for the extended safe transcript $(\tau,g)$ with
\[
q(\tau,g):=q(\tau)p(g\mid \tau).
\]

If the next query is a reward query at a prompt $x\neq x^\star$, then the observed reward is always zero, and the distribution of the sampled completion depends only on the chosen policy and not on the hidden state. Hence the same argument applies: the reply distribution is common to all surviving hidden states, so the claim is preserved and the surviving set does not change.

If the next query is a reward query at the hard prompt $x^\star$, let $\pi_\tau(\cdot\mid x^\star)$ be the policy chosen after observing $\tau$, and let $y$ be the sampled completion.

If $y\notin\mathcal Z$, then the reward is again zero for every hidden state, and the reply distribution is common to all surviving states. So the same argument applies and $C(\tau)$ is unchanged.

If $y\in\mathcal Z$ and the reward is positive, then the transcript is no longer safe and there is nothing to prove.

If $y\in\mathcal Z$ and the reward is zero, then the transcript remains safe. For a surviving hidden state $\theta$ with target completion different from $y$, the probability of this reply is $\pi_\tau(y\mid x^\star)$, which again does not depend on $\theta$. For the unique hidden state whose target completion equals $y$, the joint probability becomes zero because that state would have produced positive reward. Therefore the new safe transcript $(\tau,y,0)$ satisfies the same formula with
\[
q(\tau,y,0):=q(\tau)\pi_\tau(y\mid x^\star),
\]
and the new surviving set is $C(\tau)\setminus\{y\}$.

This completes the induction.
\end{proof}

The symmetry lemma implies that reward queries can only locate the target by direct hits, and that every negative reward at the hard prompt eliminates at most one candidate.

\begin{lemma}
\label{lem:hit-prob-dp}
Let $n\ge 1$ and $r\ge 0$. Consider a state in which the hidden target is uniform over a surviving candidate set $C$ of size $n$, and the learner has at most $r$ reward queries remaining. Then, regardless of how it may interleave future generator queries with those reward queries, the probability of ever receiving positive reward is at most $r/n$.
\end{lemma}

\begin{proof}
We prove the claim by induction on $r$.

If $r=0$, there are no reward queries left, so the probability of ever receiving positive reward is $0$. This is exactly $r/n$.

Now suppose $r\ge 1$ and that the claim holds for $r-1$. Condition on the current state, where the target is uniform over a set $C$ with $|C|=n$.

Before the next reward query, the algorithm may make any number of generator queries. On the event $A^c$, such generator queries never reach the scaffold and therefore do not change the posterior-symmetry structure from Lemma~\ref{lem:posterior-symmetry}: the hidden target remains uniform over the same surviving set $C$. So only the next reward query matters.

Let the next reward query at the hard prompt sample a completion according to some policy. Write $p_y$ for the probability of sampling candidate $y\in C$, and let
\[
p_{\mathrm{out}}:=1-\sum_{y\in C} p_y
\]
be the total probability of sampling a completion outside $C$.

If the sampled completion lies outside $C$, then the reward is certainly zero and no candidate is eliminated. The learner still has at most $r-1$ reward queries remaining, and by the induction hypothesis the future chance of success is at most
\[
\frac{r-1}{n}.
\]

If the sampled completion is some $y\in C$, then there are two possibilities. With probability $1/n$, the hidden target is exactly $y$, and the learner receives positive reward immediately. With the complementary probability $(n-1)/n$, the hidden target lies in $C\setminus\{y\}$, the reward is zero, and the candidate $y$ is eliminated. At that point there are at most $r-1$ reward queries remaining and $n-1$ surviving candidates, so the induction hypothesis bounds the future success probability by
\[
\frac{r-1}{n-1}.
\]

Combining the two cases gives
\begin{align*}
\Pr(\text{eventual positive reward})
&\le
p_{\mathrm{out}}\cdot \frac{r-1}{n}
+
\sum_{y\in C} p_y
\left(
\frac{1}{n}
+
\frac{n-1}{n}\cdot \frac{r-1}{n-1}
\right)\\
&=
p_{\mathrm{out}}\cdot \frac{r-1}{n}
+
\sum_{y\in C} p_y \cdot \frac{r}{n}\\
&\le
\left(1-\sum_{y\in C}p_y\right)\frac{r-1}{n}
+
\left(\sum_{y\in C}p_y\right)\frac{r}{n}\\
&\le
\frac{r}{n}.
\end{align*}
This completes the induction.
\end{proof}

\begin{lemma}
\label{lem:reward-hit}
Conditional on the event $A^c$, the probability that some reward query ever receives positive reward is at most $q_r/N$.
\end{lemma}

\begin{proof}
On the event $A^c$, every generator query avoids the scaffold. By Lemma~\ref{lem:posterior-symmetry}, after every safe transcript the hidden target remains uniform over the current surviving candidate set. At the start of the reward phase, the surviving set is the full set $\mathcal Z$ of size $N$.

Now ignore the generator queries and look only at the evolution of the surviving candidate set across reward queries. Since the target is initially uniform over $N$ candidates and there are at most $q_r$ reward queries in total, Lemma~\ref{lem:hit-prob-dp} applied with $n=N$ and $r=q_r$ gives
\[
\Pr(\text{some reward query receives positive reward}\mid A^c)\le \frac{q_r}{N}.
\]
\end{proof}

\begin{lemma}
\label{lem:final-mass}
Conditional on the event $A^c$ and on the event that no reward query ever receives positive reward, we have
\[
\Pr\!\left(
\widehat\pi(z_{S,B}\mid x^\star)>\frac{1}{4}
\ \middle|\
A^c,\ \text{no reward hit}
\right)
\le
\frac{4}{N-q_r}.
\]
\end{lemma}

\begin{proof}
Fix any final safe transcript $\tau$. By Lemma~\ref{lem:posterior-symmetry}, the posterior on $\Theta$ given $\tau$ is uniform over the surviving set $C(\tau)$. Each negative reward query at the hard prompt can eliminate at most one candidate from $\mathcal Z$, and generator queries eliminate none. Hence
\[
|C(\tau)|\ge N-q_r.
\]

Let $\widehat\pi_\tau(\cdot\mid x^\star)$ be the policy output by the algorithm after observing $\tau$. Conditional on $\tau$, the posterior expected mass that this policy assigns to the true target is
\[
\mathbb E\bigl[\widehat\pi_\tau(z_{S,B}\mid x^\star)\mid \tau\bigr]
=
\frac{1}{|C(\tau)|}\sum_{z\in C(\tau)} \widehat\pi_\tau(z\mid x^\star)
\le \frac{1}{|C(\tau)|}
\le \frac{1}{N-q_r}.
\]
Markov's inequality therefore implies
\[
\Pr\!\left(
\widehat\pi_\tau(z_{S,B}\mid x^\star)>\frac{1}{4}
\ \middle|\ \tau
\right)
\le
\frac{4}{N-q_r}.
\]
Since the bound is uniform over all final safe transcripts $\tau$, averaging over $\tau$ gives the claimed inequality.
\end{proof}

\subsection{Proof of the lower bound}

\begin{theorem}
\label{thm:app-bridge-lower}
Assume
\[
R=\beta\log\!\left(\frac{4}{q_0}\right).
\]
Consider any algorithm that may interleave at most $q_g$ no-reset generator queries with at most $q_r<N$ outcome-reward queries and then outputs a prompt-conditioned policy $\widehat\pi$. Then there exists a choice of hidden suffix $s\in\Sigma^L$ and reward bit $b\in\{0,1\}$ such that
\[
\Pr\!\left(
J_\beta(\widehat \pi)\ge J_\beta(\pi^\star_{s,b})-\frac{\eta\beta}{4}
\right)
\le
q_g p_+^D + \frac{q_r}{N} + \frac{4}{N-q_r}.
\]
\end{theorem}

\begin{proof}
Run the algorithm against a uniformly random hidden state
\[
\Theta=(S,B),
\]
and let $\pi^\star_\Theta:=\pi^\star_{S,B}$ be the corresponding exact optimizer. Define the success event
\[
\mathsf{Succ}:=
\left\{
J_\beta(\widehat\pi)\ge J_\beta(\pi^\star_\Theta)-\frac{\eta\beta}{4}
\right\}.
\]
By Corollary~\ref{cor:average-gap}, on every instance the event $\mathsf{Succ}$ implies
\[
\widehat\pi(z_{S,B}\mid x^\star)>\frac{1}{4}.
\]

Therefore
\[
\mathsf{Succ}
\subseteq
A
\cup
\{\text{some reward query receives positive reward}\}
\cup
\left\{
A^c,\ \text{no reward hit},\ \widehat\pi(z_{S,B}\mid x^\star)>\frac{1}{4}
\right\}.
\]
Taking probabilities and applying the union bound gives
\begin{align*}
\Pr(\mathsf{Succ})
&\le
\Pr(A)
+
\Pr(\text{some reward query receives positive reward},A^c)\\
&\qquad
+
\Pr\!\left(
A^c,\ \text{no reward hit},\ \widehat\pi(z_{S,B}\mid x^\star)>\frac{1}{4}
\right).
\end{align*}

We now bound the three terms. By Lemma~\ref{lem:scaffold-hit},
\[
\Pr(A)\le q_g p_+^D.
\]
By Lemma~\ref{lem:reward-hit},
\[
\Pr(\text{some reward query receives positive reward},A^c)\le \frac{q_r}{N}.
\]
By Lemma~\ref{lem:final-mass},
\[
\Pr\!\left(
A^c,\ \text{no reward hit},\ \widehat\pi(z_{S,B}\mid x^\star)>\frac{1}{4}
\right)
\le
\frac{4}{N-q_r}.
\]
Combining the three bounds yields
\[
\Pr(\mathsf{Succ})\le q_g p_+^D+\frac{q_r}{N}+\frac{4}{N-q_r}.
\]

This is the average success probability of the algorithm under the uniform prior on the hidden state. Therefore there exists at least one fixed instance $(s,b)$ whose success probability is at most the same quantity. That instance is the one claimed in the theorem.
\end{proof}

\begin{corollary}
\label{cor:hard-prompt-embedding}
Let $\rho$ be any prompt distribution with $\rho(x^\star)=\eta>0$. There is a family of KL-regularized outcome-reward post-training instances that is trivial on every prompt $x\neq x^\star$ and satisfies the lower bound from Theorem~\ref{thm:app-bridge-lower}. In particular, any guarantee of average regret below $\eta\beta/4$ must solve the hard prompt $x^\star$.
\end{corollary}

\begin{proof}
This is exactly the hard family defined at the start of the appendix section. Theorem~\ref{thm:app-bridge-lower} applies directly, and the final statement is just Corollary~\ref{cor:average-gap} restated in prompt-distributed language.
\end{proof}

\begin{corollary}
\label{cor:app-constant-reward-gap}
Fix $K\ge 2$, $\lambda>0$, $\beta>0$, and any constant $q_r\ge 1$. There exists a constant $c=c(K,\lambda)>1$ such that for every sufficiently large horizon $H$ and every prompt distribution $\rho$ with a designated prompt $x^\star$ of mass $\eta>0$, there is a family of KL-regularized outcome-reward post-training instances of horizon $H$ with the following property. With chosen-prefix next-token sampling access, one reward query and
\[
O\!\left(H\log\!\frac{H}{\delta}\right)
\]
generator queries suffice to output the exact optimizer with probability at least $1-\delta$. With only no-reset generator access, any algorithm that uses at most $q_r$ reward queries and succeeds with probability at least $2/3$ in outputting a policy of average regret at most $\eta\beta/4$ must make
\[
\Omega(c^H)
\]
generator queries.
\end{corollary}

\begin{proof}
Fix the horizon $H$, and choose
\[
D:=\left\lfloor \frac{H-1}{2}\right\rfloor,
\qquad
L:=H-D-1.
\]
Then both $D$ and $L$ are $\Theta(H)$, and the hard family above has exactly horizon $H$.

The upper bound follows directly from Theorem~\ref{thm:app-bridge-upper}. Since $L\le H$ and $\Delta>0$ depends only on $(K,\lambda)$,
\[
H+L\left\lceil \frac{2}{\Delta^2}\log\!\left(\frac{L(K-1)}{\delta}\right)\right\rceil
=
O\!\left(H\log\!\frac{H}{\delta}\right).
\]

For the lower bound, recall that $N=2K^L$. Since $q_r$ is fixed and $L=\Theta(H)$, we have
\[
\frac{q_r}{N}+\frac{4}{N-q_r}\to 0
\qquad\text{as } H\to\infty.
\]
So for all sufficiently large $H$,
\[
\frac{q_r}{N}+\frac{4}{N-q_r}\le \frac{1}{6}.
\]
Now suppose an algorithm succeeds with probability at least $2/3$ while using at most $q_r$ reward queries. By Theorem~\ref{thm:app-bridge-lower}, there exists an instance for which
\[
\frac{2}{3}\le q_g p_+^D+\frac{q_r}{N}+\frac{4}{N-q_r}\le q_g p_+^D+\frac{1}{6}.
\]
Hence
\[
q_g p_+^D\ge \frac{1}{2},
\qquad\text{so}\qquad
q_g\ge \frac{1}{2}p_+^{-D}.
\]

Since $D=\Theta(H)$ and $0<p_+<1$, the quantity $p_+^{-D}$ grows exponentially in $H$. For example, with
\[
c:=p_+^{-1/2}>1,
\]
we have
\[
p_+^{-D}=\Omega(c^H).
\]
Therefore
\[
q_g=\Omega(c^H),
\]
which proves the corollary.
\end{proof}

\section{Extended discussion of related work}
\label{app:related}

We expand here on the related work from Section~\ref{sec:related}.

The closest theoretical precursor is the recent line of work that makes coverage central under fixed sampling-style access. Foster, Mhammedi, and Rohatgi study exploration with language models when the learner interacts through a sampling oracle, and show that insufficient coverage of near-optimal responses lower-bounds computational efficiency \cite{FosterMhammediRohatgi2025}. Huang et al.\ analyze inference-time alignment and again find that coverage governs the quality-compute tradeoff of sampling-based alignment methods \cite{HuangEtAl2025BestOfN}. Chen et al.\ elevate this into a broader coverage principle for understanding why pre-training enables post-training and inference-time scaling \cite{ChenEtAl2025Coverage}. Huang et al.\ study sharpening-style self-improvement under a fixed access model and ask when expensive verification can be amortized into a stronger policy \cite{HuangEtAl2024Sharpening}. Our results agree with these papers where the access model overlaps with ours. The difference is the question: we vary the generator interface itself and show that on-policy reachability is the governing quantity of the no-reset class rather than a universal law of post-training.

The second conceptual precursor comes from reinforcement learning and conditional-access models for sequential distributions. Reset-style interaction can change what is learnable or plannable in RL \cite{MhammediFosterRakhlin2024,RohatgiFoster2025}. Conditional queries can likewise make otherwise hard sequential models tractable to learn \cite{MahajanKakadeKrishnamurthyZhang2023,LiuMoitra2024}. The prefix-tree formalism in the present paper is the autoregressive version of the same distinction. Root-start interaction corresponds to ordinary episodic access. Chosen-prefix interaction corresponds to reset-like access. The point of the paper is that, for generators, this distinction should be native rather than imported only by analogy.

There is also a substantial literature on richer intermediate signals that are not generator-side in our sense. Botta et al.\ study verifier-guided generation with prefix-level feasibility signals \cite{BottaLiMehtaAshZhangRisteski2025}. Amani et al.\ study adaptive partial rationale exposure \cite{AmaniLotfiBaldwinBengioFarajtabarAbbeWest2025}. Tuyls et al.\ study hidden-state-based exploration bonuses \cite{TuylsFosterKrishnamurthyAsh2025}. These papers all support the broader intuition that intermediate access matters, but the channel is different. The present paper isolates the stochastic generator itself and asks what can be learned from distributional queries without hidden states or external verifiers.

Finally, the KL bridge theorem is meant to be complementary to feedback-side theory. Mousavi-Hosseini and Erdogdu show that process rewards can avoid outcome-reward barriers under fixed generator access \cite{MousaviHosseiniErdogdu2026}. The theorem in Section~\ref{sec:kl-bridge} holds the feedback model fixed and varies only the generator interface. In that sense it fills a missing quadrant: fixed generator access with varying feedback has one kind of barrier, while fixed feedback with varying generator access has another.

The empirical literature points in the same direction. Yue et al.\ argue that RL with verifiable rewards often sharpens trajectories already present in the base model \cite{YueEtAl2025RLVR}. Karan and Du show that stronger sampling can sometimes match or exceed RL post-training while preserving diversity \cite{KaranDu2025Sampling}. Tan et al.\ identify entropy collapse in the final-layer posterior of reasoning models and recover exploration from intermediate-layer uncertainty \cite{TanEtAl2026LED}. These papers do not prove the taxonomy in this paper, but they reinforce its premise: support, exploration, and intermediate control should be treated as part of the problem rather than as implementation details.


\newpage

\end{document}